\documentclass[runningheads]{llncs}

\usepackage{eccv}

\usepackage{eccvabbrv}

\usepackage{graphicx}
\usepackage{booktabs}

\usepackage[accsupp]{axessibility}  

\usepackage{hyperref}

\usepackage{orcidlink}

\begin{document}

\title{Snuffy: Efficient Whole Slide Image Classifier} 

\titlerunning{Snuffy}

\author{Hossein Jafarinia\inst{1}\orcidlink{0009-0003-4172-8372} \and
Alireza Alipanah\inst{1}\orcidlink{0009-0000-1292-9296} \and 
Danial Hamdi\inst{2}\orcidlink{0009-0005-1419-3338} \and
Saeed Razavi\inst{1}\orcidlink{0009-0009-0760-0758} \and 
Nahal Mirzaie\inst{1}\orcidlink{0009-0003-6954-7151} \and 
Mohammad Hossein Rohban\inst{1}\orcidlink{0000-0001-6589-850X} \thanks{Corresponding author.}}

\authorrunning{H. Jafarinia et al.}

\institute{Sharif University of Technology, Tehran, Iran \\
\email{\{jafarinia, alireza.alipanah46, saeed.razavi, nahal.mirzaie, rohban\}@sharif.edu}\\
\url{https://www.sharif.edu}
\and Amirkabir University of Technology (Tehran Polytechnic), Tehran, Iran \\
\email{danial.hamdi@outlook.com}\\
\url{https://aut.ac.ir}
}

\maketitle

\begin{abstract}
\small
Whole Slide Image (WSI) classification with multiple instance learning (MIL) in digital pathology faces significant computational challenges. Current methods mostly rely on extensive self-supervised learning (SSL) for satisfactory performance, requiring long training periods and considerable computational resources. At the same time, no pre-training affects performance due to domain shifts from natural images to WSIs.
We introduce \textbf{\textit{Snuffy}} architecture, a novel MIL-pooling method based on sparse transformers that mitigates performance loss with limited pre-training and enables continual few-shot pre-training as a competitive option. Our sparsity pattern is tailored for pathology and is theoretically proven to be a universal approximator with the tightest probabilistic sharp bound on the number of layers for sparse transformers, to date. We demonstrate Snuffy's effectiveness on CAMELYON16 and TCGA Lung cancer datasets, achieving superior WSI and patch-level accuracies. The code is available on \url{https://github.com/jafarinia/snuffy}.

\noindent
\begin{minipage}{\linewidth}
\captionsetup{type=figure,hypcap=false} 
\centering
  \includegraphics[width=0.93\textwidth]{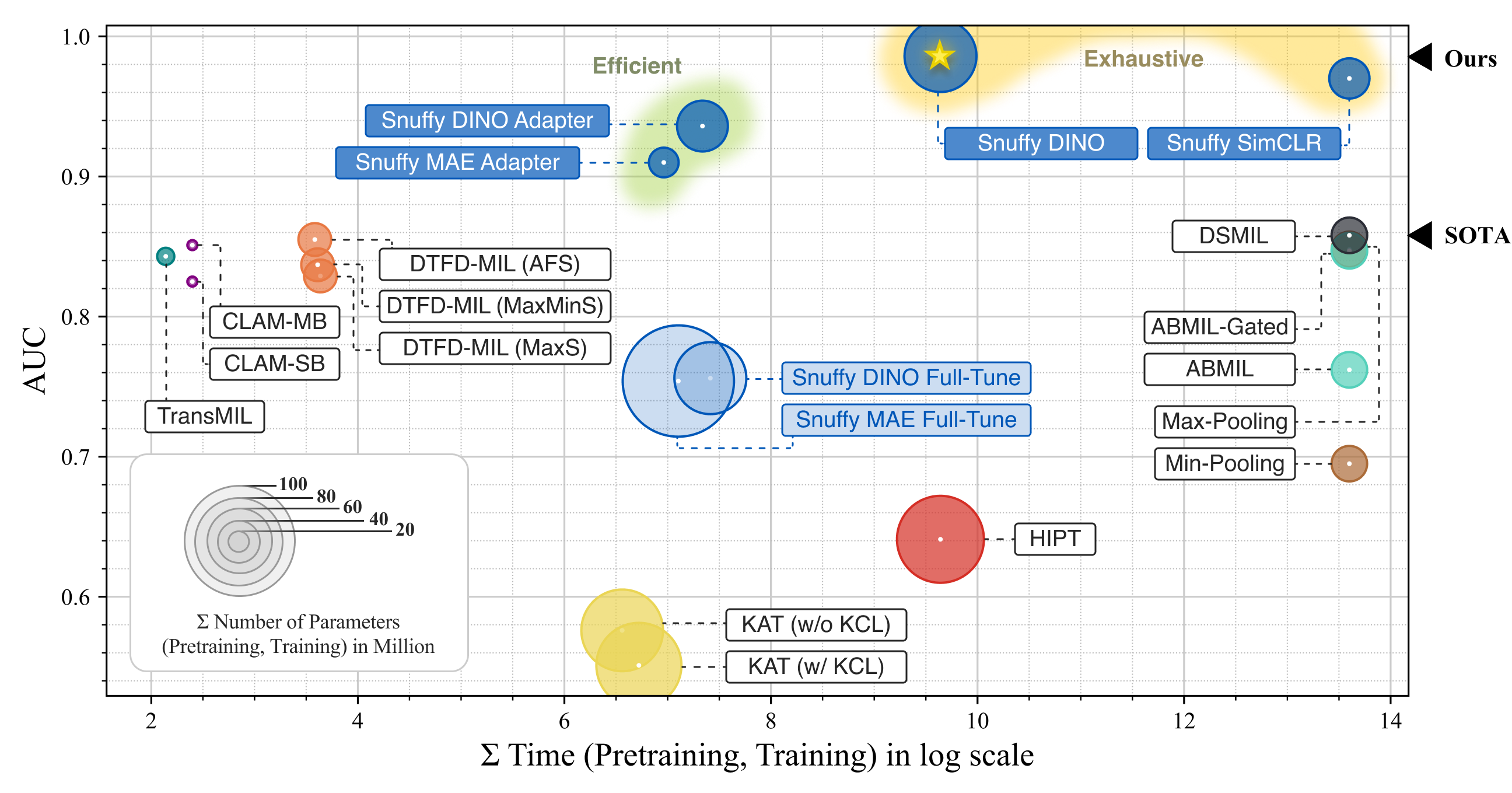}
  \captionof{figure}{Performance (AUC) vs. efficiency (size and time) trade off on CAMELYON16.}\label{fig:rebut}
\end{minipage}

\end{abstract}


\section{Introduction}
\label{sec:intro}

The emergence of whole slide images (WSIs) has presented significant opportunities to leverage machine learning techniques for essential tasks of cancer diagnosis and prognosis \cite{cellcommunitydetect, multiclasstexture, biomarkeverything}. Nevertheless, integrating contemporary deep learning advancements into these areas faces notable challenges. Primarily, the sheer size of WSIs, with usual dimensions of approximately 150,000 \(\times\) 150,000 pixels, renders them unmanageable for processing and training with existing deep learning frameworks on current hardware \cite{pathtoclinic}.

A common strategy to address this challenge is to divide WSIs into smaller patches followed by the application of Multiple Instance Learning (MIL) \cite{pathtoclinic,dnnforcomputationalhistology,notsosupervised,dependenciesindeepmil,deepweaksurvey,milinclasssegclus}. MIL, a variant of weakly supervised learning, considers instances as elements of sets involving an embedding phase and a pooling operation (MIL-pooling). The embedding phase often employs a pre-trained vision backbone, frequently with a self-supervised learning technique applied on the patches, transforming these patches into embeddings. Subsequently, MIL-pooling aggregates these embeddings, deriving scores at both the patch and WSI levels \cite{abmil}.

Recent advancements in WSI classification have achieved significant outcomes but face challenges like data-hungriness, high computational and memory requirements. These issues hinder the deployment and development of deep learning technologies in clinical and research settings. For example, RNN-MIL \cite{milrnn} and HIPT \cite{hipt} require tens of terabytes of data, while DSMIL \cite{dsmil} and DGMIL \cite{dgmil} require several months for pre-training phases. DTFD-MIL \cite{dtfd} uses an efficient MIL-pooling strategy but demands over 100 gigabytes of system memory for training, which is only feasible in some settings. Conversely, the absence of pre-training or insufficient pre-training degrades performance because of the domain shift from natural image datasets such as ImageNet-1K to WSIs (CLAM \cite{clam} and KAT's \cite{kat} low AUC as shown in Fig. \ref{fig:rebut}).

This work presents an approach that significantly reduces the computational demands required for training the embeddings by orders of magnitude. Then, empower its expressivity in a novel MIL-pooling to compensate for performance loss due to limited pre-training. Snuffy makes continual few-shot pre-training possible and a competitive option in this field by balancing efficiency and performance.

Our framework comprises two key components. First, we propose using \textbf{self-supervised continual pre-training} with Parameter Efficient Fine Tuning (PEFT) in the pathology domain, specifically utilizing Adaptformer\cite{adaptformer} due to its effective and straightforward design. While PEFT has been applied in various fields, its use in pathology imaging is novel. Our results indicate that transitioning from natural images to histopathology is feasible, allowing us to leverage PEFT methods effectively.\\
Second, inspired by the complex biology of cancer and the importance of the tissue microenvironment in cancer detection, we introduce the \textbf{Snuffy MIL-pooling architecture}, which features a new sparsity pattern for sparse transformers. We demonstrate that the Snuffy sparsity pattern acts as a universal approximator, with the number of layers constrained to a linear relationship with the number of patches, denoted as $\mathcal{O}(n)$. This finding represents the tightest probabilistic bound on the number of layers for sparse transformers to date.

We introduce two families within our framework: Efficient Snuffy and Exhaustive Snuffy. The Efficient Snuffy family is trained initially on a natural image dataset and then continues training with PEFT on WSIs. In contrast, the Exhaustive Snuffy family is trained from scratch on WSIs. Both families utilize the Snuffy MIL-pooling architecture. Although the performance of Efficient Snuffy may be slightly inferior to Exhaustive Snuffy, both methods significantly outperform existing benchmarks in Region-of-Interest (ROI) detection and WSI classification, setting a new state-of-the-art (SOTA).

In summary, our main contributions are as follows:
\begin{itemize}
\item Using continual SSL pre-training from ImageNet-1K pre-trained models to pathology datasets employing Adapters, substantially reducing the computational time for pre-training by an order of magnitude.

\item Introduction of a novel biologically driven sparsity pattern with a new probabilistic sharp bound on the number of layers to guarantee its universal approximation.

\item Achieving significant improvements in WSI classification metrics for both the exhaustive and efficient families, and reaching new state-of-the-art scores in WSI classification (AUC of 0.987) and ROI detection (FROC of 0.675) by the exhaustive family.

\item Validation of our method on widely recognized datasets, including CAMELYON16 and TCGA Lung Cancer WSI datasets, as well as on three classical multiple instance learning datasets: Musk1, Musk2, and Elephant, demonstrating consistent and superior performance across these varied datasets and establishing our MIL-pooling architecture as a general MIL-pooling architecture.

\end{itemize}

\section{Related Work}
\label{sec:related_work}

\subsubsection{Parameter-Efficient Fine-Tuning for Vision}
\label{peft}

Parameter-Efficient Fine-Tuning (PEFT), initially successful in NLP \cite{adapter, lora}, especially with Transformers \cite{adapter, madx, lora, unified_nlp_peft}, has recently been applied to Vision Transformers (ViTs) for supervised tasks \cite{adaptformer, medsam_adapter, clipadapter}. Techniques like Adapters \cite{adapter} and LoRA \cite{lora} help mitigate overfitting during fine-tuning. The high computational demands for self-supervised pre-training, along with SSL pre-training benefits on domain-specific datasets \cite{dsmil}, highlight PEFT's potential in SSL contexts. Recent studies \cite{pecop, fffff} have proposed continual pre-training from general datasets like ImageNet-1K to domain-specific ones. Our study is the first to apply this approach specifically from ImageNet-1K to pathology datasets, advancing the field.

\subsubsection{MIL for WSI Classification}
\label{mil_for_wsi}

The MIL-pooling operator must be permutation-invariant \cite{abmil}. Traditional methods like Mean-pooling and Max-pooling have had some success, but parameterized strategies are more effective \cite{abmil, dsmil}. Recent SOTA MIL-pooling techniques are mainly attention-based, transformer-based, and distribution-based.

\textbf{Attention-based Methods:} ABMIL \cite{abmil} uses a trainable weighted average for each patch, incorporating nonlinear functions of tanh(.) and sigm(.) and learnable parameters. DSMIL \cite{dsmil} employs a unique self-attention framework focusing on one patch's interaction with the rest, derived from Max-pooling.

\textbf{Transformer-based Approaches:} TransMIL \cite{transmil} uses a ViT variant with the Nyström method \cite{nystromformer} for self-attention approximation and introduces the Pyramid Position Encoding Generator (PPEG) for positional information. HIPT \cite{hipt} uses a three-tier hierarchical pyramid with DINO \cite{dino} in the initial stages and a compact ViT in the final stage for WSI-level tasks. KAT \cite{kat} introduces kernels reflecting spatial information across various scales for cross-attention with patch embeddings, they described it as a linear memory complexity Transformer.

\textbf{Distribution-based Methodologies:} They often use a clustering layer to improve embeddings for MIL-Pooling. CLAM \cite{clam} enhances embeddings by incorporating clustering atop its attention mechanism, while DGMIL \cite{dgmil} uses patch embedding distribution to create pseudo-labels for patches. This involves training a linear projection head and a linear layer on these pseudo-labels, utilizing the projection head for better embeddings with Mean-pooling for WSI classification. For identifying ROIs, it uses distribution scores from testing patches on these refined embeddings for patch labeling.

While these methods are effective for classification, their ROI detection performance is lacking or absent \cite{transmil, hipt, dtfd, kat}. Our research aims to improve both WSI classification and localization accuracy while demonstrating that our architecture achieves the universal approximation of self-attention mechanisms with a sharp concentration rate.

\section{Background}
\label{background}

\subsection{Notation}
For any positive integer $a$, we represent the set as $[a] = \{1,2,...,a\}$. If we have a matrix $A \in \mathbb{R}^{d \times n}$, we refer to its $j$-th column as $A_j$, and $A_S$ denotes the submatrix comprising columns of $A$ indexed by $\subseteq [n]$. The \(softmax\) operator $\sigma_S[\cdot]$ processes each column of a matrix, resulting in a column stochastic matrix.

In terms of asymptotic behavior, $f = \mathcal{O}(g)$ implies that there exists a constant $C > 0$ such that $f(n) \leq Cg(n)$ for all but finitely many $n$. Conversely, $f = \Omega(g)$ signifies that $g = \mathcal{O}(f)$. We use $\tilde X$ to conceal poly-logarithmic factors of $X$, like $\tilde\Omega(n)$.

\subsection{MIL Formulation}
\label{mil_formulation}

In a binary classification problem, the dataset \(D = \{(X_1, Y_1), ..., (X_n, Y_n)\}\) consists of bags \(X\), where each bag \(X = \{x_1, ..., x_k\}\) contains instances \(x\), and \(Y_i \in \{0, 1\}\) represents the bag's label. The individual instance labels \(\{y_{1}, ..., y_{k}\}\) with \(y \in \{0, 1\}\), are unknown during training. This is modeled as:

\begin{equation}
    Y =
    \begin{cases}
        0, & \text{iff } \sum_{i} y_i = 0\\
        1, & \text{otherwise}.
    \end{cases}
\end{equation}

or equivalently:

\begin{equation}
    Y = \max_{i} \{ y_i \}.
\end{equation}

To address the complexities in learning, to navigate this, it is proposed to train the MIL model by optimizing the log-likelihood function:

\begin{equation}
P(Y | X) = \theta(X)^Y (1 - \theta(X))^{1-Y},
\end{equation}

where \(\theta(X) \in [0, 1]\) is the probability of \(Y = 1\) given \(X\).

Given that MIL assumes no ordering or dependency of instances, \(\theta(X)\) must be a permutation-invariant (symmetric) function. This is achieved through the Fundamental Theorem of Symmetric Functions, with monomials \cite{deepsets} and a similar Theorem by \cite{pointnet} leading to:

\begin{equation}
    \theta(X) = g(\pi(f(x_1), ... , f(x_k))),
\end{equation}

where \(f\) and \(g\) are continuous transformations, and \(\pi\) is a permutation-invariant function (MIL-pooling). There are two approaches, the instance-level approach where \(f\) is an instance classifier and \(g\) is identity function, and the embedding-level approach, where \(f\) is a feature extractor and \(g\) maps its input to a bag classification score. The embedding-based approach is preferred due to its superior performance \cite{abmil}.

In Deep MIL, \(f\) typically uses pre-trained vision backbones to extract features from bag instances \cite{clam, dsmil, transmil, dgmil, dtfd, kat}. The aggregation function \(\pi\) ranges from non-parametric methods like max-pooling to parametric ones using attention mechanisms, as detailed in Section \ref{mil_for_wsi}. Finally, \(g\) is often implemented as a linear layer with \(\sigma\) to project aggregated instance representations into a bag score.

For multiclass classification, \(g\)’s output dimension is adjusted to the number of classes, and the \(softmax\) function is used instead of \(\sigma\) to distribute probabilities across classes.

\subsection{Sparse Transformers}
\label{unif_framework}

In the full-attention mechanism, each attention head considers interactions between every patch in an image. This requires significant memory and computational resources, with complexity scaling quadratically with the number of patches. However, in a sparse attention mechanism, each attention head only attends to a particular subset of patches instead of the entire set of patches. \\
Drawing upon the formalism presented in \cite{onconc}, the $i$th sparse attention head output for a patch $k$ in the layer $l$ is articulated as follows:
\begin{equation}\label{sparshead}
    SHead^{i,l}(X)_k = W_V^i X_{\mathcal{A}_k^l} \cdot \sigma \left(\left(W_K^i X_{\mathcal{A}_k^l}\right)^T W_Q^i X_k\right)
\end{equation}
When calculating the attention scores, the query vector of the $i$th head, $W_Q^iX_k$ of the $k$th patch interacts exclusively with the key vectors $W_K^iX_{\mathcal{A}_k^l}$ from patches belonging to its specific subset, ${\mathcal{A}_k^l}$. This means that attention is computed only between the $k$th patch and the patches in its assigned subset. 
Consequently, the output of each attention head for the patch $k$, $SHead^{i,l}(X)_k$ is a result of combining columns from the patch representations $W_V^i X_{\mathcal{A}^l_k}$ within its assigned subset, rather than considering the entire sequence of patches \cite{onconc}.
The collection of these subsets ${\mathcal{A}_k^l}$, across all layers $l \in [L]$ and patches $k \in [n]$, is termed the {\it sparsity patterns}.

Sparse attention mechanisms significantly reduce the time and memory complexities of transformers. However, this efficiency comes at the expense of loss of accuracy in attention matrix predictions depending on how sparse the sparsity patterns are. 
Expanding beyond this limitation, prior studies have identified diverse sparse patterns with $\mathcal{O}(n)$ connections (compared to $\mathcal{O}(n^2)$ in full-attention mechanisms), effectively approximating any full attention matrices \cite{strans,startrans,bigbird,Longformer}. Notably, \cite{onconc} established adequate conditions to ensure that any collection of sparsity patterns adhering to these criteria, alongside a probability map such as softmax, can serve as a universal approximator for sequence-to-sequence functions (Theorem~\ref{three}).

\begin{theorem}\label{three}
  A sparse transformer with any set of sparsity patterns $\{\mathcal{A}_k^l\}$ satisfying these conditions:
  \begin{enumerate}
    \item $\forall k \in [n],  \forall l \in [L], \ k \in \{\mathcal{A}_k^l\}$\label{three1} 
    \item There is a permutation $\gamma$ such that 
    $\forall i \in [n-1], \ \gamma(i) \in \bigcup\limits_{l=1}^{p} \{\mathcal{A}_{\gamma(i+1)}^l\}$.\label{three2}
    \item Each patch attends to every other patch, directly or indirectly.\label{three3}
  \end{enumerate}
  coupled with a probability map generating a column stochastic matrix that closely approximating hardmax operator, is a universal approximator of sequence-to-sequence functions \cite{onconc}. 
\end{theorem}

Put simply, criterion 1 mandates that each patch attends to itself. This condition guarantees that even if contextual attention embeddings are not computed within the self-attention framework for a patch, the patch's original embedding remains intact in the output. Additionally, criteria 2 and 3 state that when all patterns in a set of sparsity patterns are combined and an adjacency matrix is created, the resulting graph $G$, possesses both a Hamiltonian path and strong connectivity. From now on, we refer to any sparsity patterns that meet the conditions outlined in Theorem~\ref{three} as \textit{universal approximator sparsity patterns.}

\begin{figure*}
  \centering 
  \begin{subfigure}[b]{0.75\textwidth} 
    \centering 
    \includegraphics[width=\textwidth]{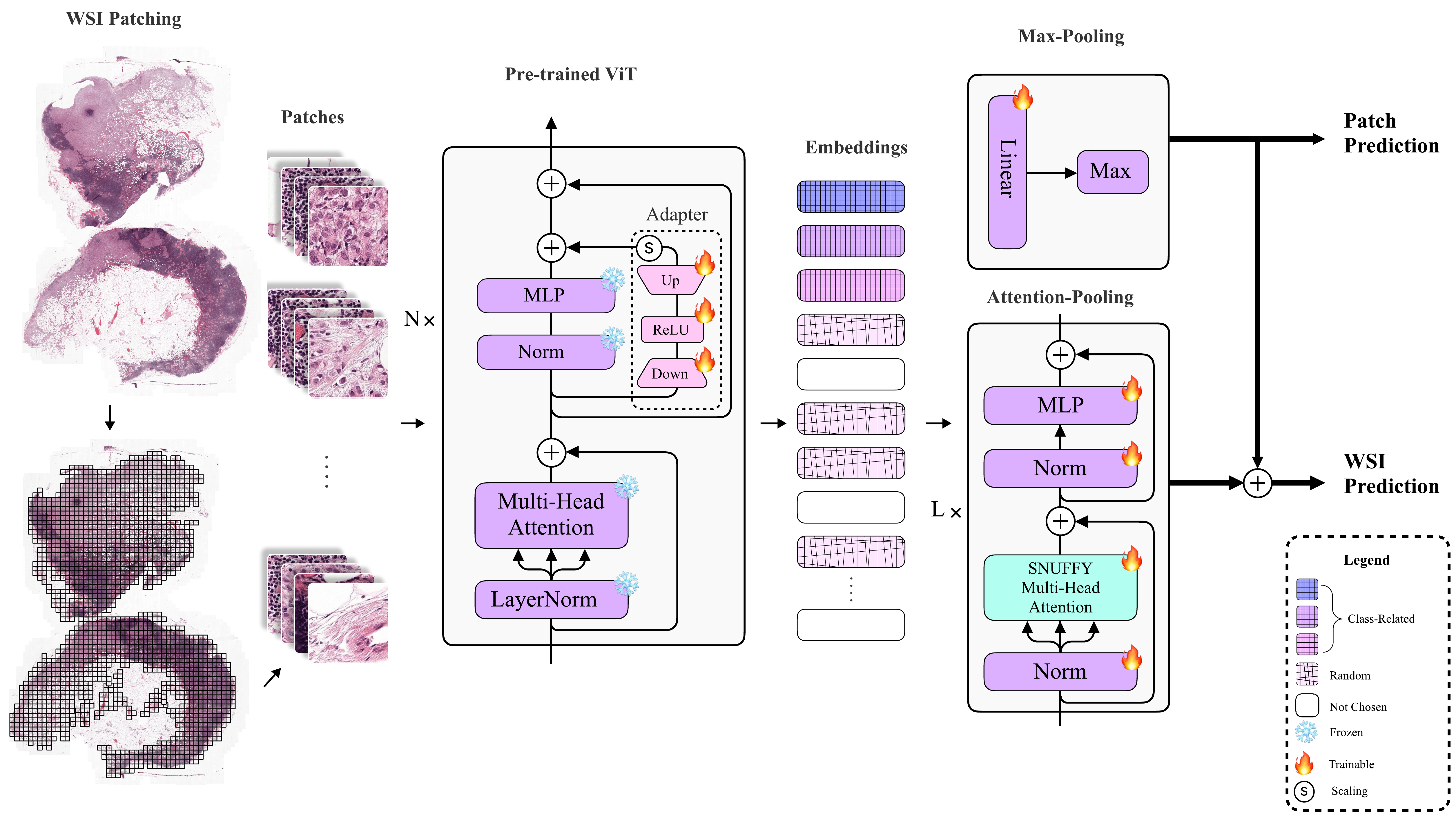}
    \caption{} \label{fig:smallarch}
  \end{subfigure}
  \hfill 
  \begin{subfigure}[b]{0.2\textwidth} 
    \centering 
    \raisebox{15mm}{
        \includegraphics[width=\textwidth]{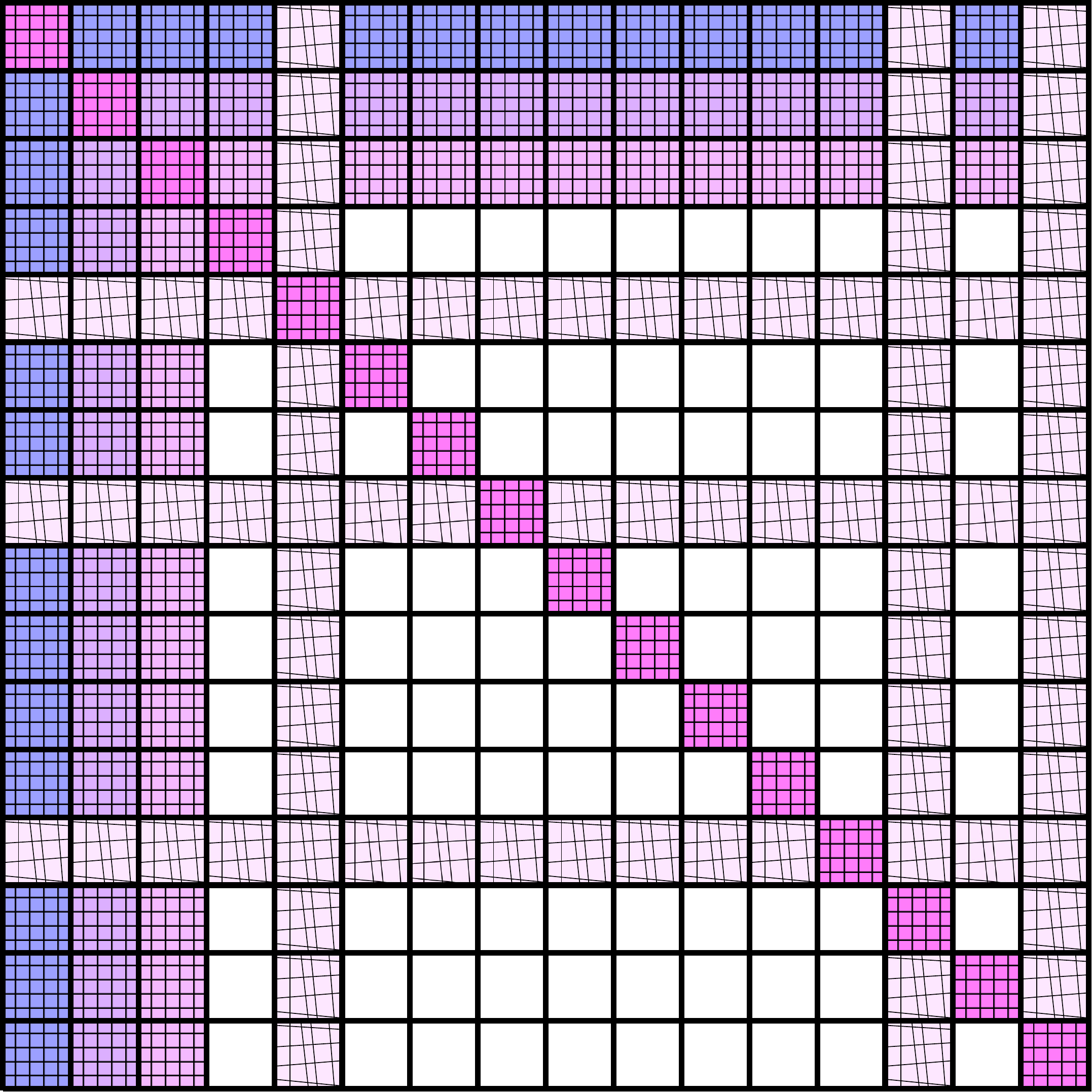}
    }
    \caption{} \label{fig:attentions}
  \end{subfigure}
  \caption{Overview of the proposed method (a) The WSIs are segmented into \(256 \times 256\) patches at 20X magnification, followed by embedding extraction via a pre-trained ViT \cite{vit}. Subsequently, these embeddings are inputted into the Snuffy for patch and WSI classification. (b) The connectivity matrix illustrates the Snuffy attention sparsity patterns, with Class-related Global Attentions, highlighted in darker colors either vertical or horizontal (the darker the more important), Diagonal Attentions depicted with pink, and Random Global Attentions shown in the lightest pink.}
\end{figure*}

\section{Method}
\label{method}

\subsection{Continual Few-shot Efficient Self-Supervised Pre-training}
\label{cont_few_eff_ssl}

To avoid extensive training on large domain-specific datasets, we propose using continual few-shot self-supervised pre-training with AdaptFormer \cite{adaptformer} on ViTs pre-trained on ImageNet-1K \cite{fffff}.

AdaptFormer \cite{adaptformer}, a PEFT variant, adds trainable weights to each layer while freezing the original network's weights. This leverages initial weights while adapting to domain-specific datasets, reducing overfitting \cite{pecop}.

Our experiments show that maintaining a self-supervised methodology during pre-training and adaptation yields a robust backbone with high-quality features, tuning fewer parameters, and reducing memory and computational demands.

For SSL method, we used Masked Autoencoder (MAE) \cite{mae}, which masks \(75\%\) of the input image and trains an encoder-decoder to reconstruct it, making the encoder a feature extractor for downstream tasks. Additionally, we used DINO \cite{dino}, known for its effectiveness in pathology \cite{hipt, benchmarking_wsi_ssl}. DINO \cite{dino} employs a student-teacher network setup where the student learns features from the teacher, who has a more global view of the image.

\subsection{Snuffy Architecture}
Informed by the prerequisites outlined for sparse transformers, pathologists' behavior, and drawing inspiration from DSMIL \cite{dsmil}, our Sparse Transformer architecture in MIL-pooling constitutes two main components: Max-pooling component and Attention-pooling (see Fig~\ref{fig:smallarch}). Further, we provide detailed descriptions of each component.

\subsubsection{Max-pooling Component}
The Max-pooling component serves a pivotal role within our framework. Initially, it functions as an effective instance-level model \cite{dtfd}.
Subsequently, it enhances the efficacy of the Attention-pooling module. On the other hand, the Attention-pooling module facilitates more effective optimization of the Max-pooling module. The empirical evidence supporting this assertion is presented in Table \ref{tablecamelyon}.

Given these observations, our methodology leverages the top \(\lambda_{top}\) patches from the Max-pooling model as the Class-related Global Attentions.

\subsubsection{Attention-pooling Component}
Present sparsity patterns are predominantly designed with a focus on NLP tasks \cite{o_n_enough, bigbird}, often exhibiting a prevalent structure known as the locality windowed pattern. However, in the domain of WSIs, pathologists frequently depend on the identification of non-related tissue within other tissues as a pivotal biomarker for cancer detection, even in instances where the tissue itself does not exhibit overt signs of malignancy \cite{skin_metastasis,paraneoplastic,dermato_asso,obs_jaundice}.\\

Inspired by the inherent characteristics of WSI analysis, we propose Snuffy sparsity patterns.

\begin{definition}\label{sparsitypattern}
\textbf{Snuffy} sparsity patterns for patch $k$ in layer $l$ are defined as:
  \begin{equation}
    A_{k}^{l} = \{k\} \cup 
    \begin{cases}
        [n] &  k \in \Lambda^l \\
        \Lambda^l  & \text{otherwise}
    \end{cases}
  \end{equation}
where $\Lambda^l := \Lambda_{top} \cup \Lambda^l_{r}$. Here, $\Lambda^l$ is set of patches selected in layer $l$ which consists of $\Lambda_{top}$ representing a set of top $\lambda_{top}$ patches from the max-pooling component, and $\Lambda_{r}^l$ denoting a set of random patches $\in [n]\backslash\Lambda_{top}$. 
\end{definition}

The Snuffy sparsity concept comprises three key elements: \textbf{Class-related Global Attentions} ($\Lambda_{top}$), \textbf{Random Global Attentions} ($\Lambda^l_{r}$), and \textbf{Diagonal Attentions} (${k}$). Class-related Global Attentions are crucial for the final classification task, and we leverage our max-pooling mechanism to identify them. We incorporate Random Global Attentions into our design to integrate insights from the tissue microenvironment, enhancing the capture of critical information. Diagonal Attentions ensure that even if contextual attention embeddings are not computed within the self-attention framework for a patch, the patch's original embedding remains preserved in the output.

After several iterations, the max-pooling component converges, and the scores for patches become fixed, allowing us to assume $\Lambda_{\text{top}}$ remains constant across layers (see illustration of Snuffy patterns in Fig~\ref{fig:attentions}).

\section{Universal Approximation of Snuffy}
\label{univ_appr}
In this section, we demonstrate that our sparse transformer serves as a universal approximator for sequence-to-sequence functions. By applying Theorem~\ref{three} and validating the Snuffy sparsity patterns defined in Definition~\ref{sparsitypattern}, we confirm that our transformer, utilizing softmax as the probability map, satisfies all conditions given in the theorem. Furthermore, we illustrate that our transformer does not necessitate $\tilde\Omega(n)$ layers, as previously suggested in studies \cite{bigbird}. Instead, it requires only $\mathcal{O}(\frac{n\log{2}}{\lambda_r})$ layers to ensure universal approximation with high probability, achieving the most stringent probabilistic limit of the layer count to our knowledge.\\

Snuffy sparsity patterns unquestionably satisfy the criteria \ref{three1} and \ref{three3} outlined in Theorem~\ref{three}. The first condition is fulfilled through the inclusion of $k \in \mathcal{A}_k^l$, as defined in Definition~\ref{sparsitypattern}. Moreover, the presence of at least one global attention patch within the patterns ensures connectivity among all patches, with a maximum path length of 2, thus meeting condition \ref{three3}. \\

To satisfy the criterion \ref{three2}, we must demonstrate the existence of a Hamiltonian path in the graph corresponding to the union of patterns in the Snuffy sparsity patterns. Initially, we introduce Proposition~\ref{graphtheory} from graph theory to facilitate the proof. We employ the proposition and demonstrate that covering half of the patches in all layers, overall satisfies its properties. This leads to the formation of a Hamiltonian path, thus fulfilling the desired proof.

\begin{proposition}\label{graphtheory}
    Every graph $G(E,V)$ with $E$ and $V$ as the set of edges and nodes, with $\rvert V\lvert \geq3$ and $\alpha(G) \leq \chi(G)$ has a Hamiltonian cycle. Where $\alpha(G)$ is the maximum independent set, and $\chi(G)$ is the chromatic number of $G$.
\end{proposition}
\begin{proof}
    see Supplementary Material~\ref{propositionproof}.
\end{proof}

\begin{lemma}\label{lemma}
    For $G_S$, the graph representing Snuffy sparsity patterns, we guarantee that there exists a Hamiltonian path if
    \begin{equation}
        \rvert \bigcup\limits_{l\in[L]} \Lambda^l\lvert \ \geq \ \frac{n-1}{2}
    \end{equation}
\end{lemma}
\begin{proof}  
The maximum independent set of patches in $G_S$ is equal to $[n]\backslash\Lambda$, where $\Lambda :=  \bigcup_{l\in[L]} \Lambda^l$ is the set of all covered patches. Conversely, the minimum number of colors needed to color the vertices of $G_S$ is $\rvert\Lambda\lvert  + 1$. Therefore, to satisfy $\alpha(G_S) \leq \chi(G_S)$, we must demonstrate that $\rvert [n]\backslash\Lambda \lvert \geq \rvert \Lambda\lvert + 1 $, which implies $\rvert \Lambda\lvert \geq \frac{n-1}{2}$.
 (for more details see supplementary~\ref{lemmaproof})
\end{proof}

Considering Lemma~\ref{lemma}, we make the keen observation that we only need to ensure that after a finite number of layers, we have covered at least $\lfloor{\frac{n}{2}}\rfloor$ of the patches. This observation resembles a generalized version of the coupon collector problem, where the goal is to collect half of the coupons in a $\lambda_r$-uniform group setting. In this scenario, at each step, we can collect $\lambda_r$ number of cards simultaneously. This leads us to our main theorem:

\begin{theorem}\label{maintheorem}
    If $\lambda_r = \mathcal{O}(n)$, where n is number of patches, and $\lambda_r = \rvert \Lambda_r \lvert$, then the number of layers $L$ needed to prove that the Snuffy sparsity pattern (defined in ~\ref{sparsitypattern}) is a universal approximator sparsity pattern is concentrated around $\frac{n\log{2}}{\lambda_r}$. More precisely, we have:
    \begin{equation}
        \lim_{n \to \infty} \mathbb{P}( \rvert L - \frac{n\log{2}}{\lambda_r} \lvert > \frac{c\sqrt{n}}{\lambda_r}) \longrightarrow 1 - \Phi(c)
    \end{equation}
\end{theorem}
\begin{proof}
    see Supplementary Material~\ref{maintheoremproof}
\end{proof}

\section{Experiments and Results}
\label{exp_res}

\subsection{Datasets}
\label{dataset}

For evaluation, we use the CAMELYON16 , TCGA Lung Cancer, and MIL Datasets \cite{camelyon16, tcga, muskmildatasets, animalmildatasets}. The CAMELYON16 dataset \cite{camelyon16} includes 270 training and 129 testing WSIs, categorized into tumor and normal classes. The TCGA Lung Cancer dataset \cite{tcga} consists of 1042 WSIs (530 LUAD and 512 LUSC). The MIL Datasets \cite{muskmildatasets, animalmildatasets} are specifically designed for MIL. CAMELYON16 is particularly challenging and referred to as a "needle-in-a-haystack" dataset \footnote{\url{https://github.com/mahmoodlab/HIPT/issues/41}}. For more details on these datasets, refer to the Supplementary Material \ref{supdatasets}.

\subsection{Experimental Setup}
\label{exp_setup}

For evaluation of the CAMELYON16 dataset \cite{camelyon16} we segmented WSIs into \(256 \times 256\) patches at \(20X\) magnification level, excluding patches predominantly featuring background elements and adhered to its official split for testing.

In the analysis of the TCGA Lung Cancer dataset \cite{tcga}, WSIs were similarly segmented into \(256 \times 256\) patches at a \(20X\) magnification, with background patches being discarded. The dataset was divided into training, validation, and test sets, constituting roughly \(60\%\), \(15\%\), and \(25\%\) of the total, respectively.

For the MIL datasets \cite{muskmildatasets, animalmildatasets}, we implemented a 10-fold cross-validation procedure.

This comprehensive and iterative evaluation approach empowers us to affirm the efficacy of our proposed framework alongside other SOTAs with substantial reliability.

We refer to Snuffy models as Snuffy + SSL pre-training method + Exhaustive for training from scratch, Adapter for fine-tuning with adapter, and Full-tune for fine-tuning all weights without adapter.

For more on the experimental setup and imeplementation details see Supplementary Materials \ref{exp_setup_supp}, \ref{imp_details} respectively.

\subsection{Evaluation Metrics}
\label{eval_metrics}

For WSI classification across the CAMELYON16 \cite{camelyon16}, TCGA Lung Cancer \cite{tcga}, and MIL datasets \cite{muskmildatasets, animalmildatasets}, we utilize Accuracy and Area Under the Receiver Operating Characteristic Curve (AUC) as standard evaluation metrics. Specifically, for the CAMELYON16 dataset \cite{camelyon16}, recognized for its complexity, we underscore the critical yet mostly overlooked aspect of model calibration evaluation in the high-stakes domain of WSI classification, which directly impacts patient care. In this context, we utilize the evaluation of Expected Calibration Error (ECE), the most widely adopted calibration measure. ECE represents the average discrepancy between the model's prediction confidence and actual performance accuracy. For more details check Supplementary Material \ref{calibration}.

For ROI detection, exclusively applicable to the CAMELYON16 dataset, we employ the Patch classification AUC and the Free Response ROC (FROC) as the benchmark metrics for evaluating ROI detection within WSI classification frameworks. The Patch classification Accuracy is omitted from our reporting due to the unbalanced portion of tumor regions in this dataset.

\subsection{Baselines}
\label{baseline}

For baseline we employ traditional approaches, such as Max-pooling and Mean-pooling, and current SOTAs including ABMIL, ABMIL-Gated \cite{abmil}, CLAM-SB, CLAM-MB \cite{clam}, non-local DSMIL \cite{dsmil}, TransMIL \cite{transmil}, DGMIL \cite{dgmil}, HIPT \cite{hipt}, DTFD-MIL (MaxS), DTFD-MIL (MaxMinS), DTFD-MIL (AFS) \cite{dtfd}, KAT (w/o KCL), KAT (w/ KCL) \cite{kat}. If a model does not provide a direct way of ROI detection, we do not report the corresponding metrics for it. Due to computational resource constraints, we could not directly evaluate DGMIL \cite{dgmil} across all datasets, nor could we assess HIPT \cite{hipt} on the TCGA Lung Cancer dataset \cite{tcga}, as these methods need highly resource-intensive vision backbone pre-training. For Slide and Patch metrics, We have referred to the performance metrics reported in their original publications for these methods.

\subsection{Results and Analysis}
\label{res_analysis}

As shown in Table \ref{tablecamelyon}, Snuffy achieves competitive performance in WSI classification and patch classification, using minimal resources. When exhaustive SSL pre-training is applied, Snuffy consistently outperforms other models. Its strong performance in the calibration metric ECE highlights its potential for clinical applications. Snuffy's superiority is further validated by the results in Table \ref{tabletcga}. Also, Adapter vs Full-tune shows fine-tuning with Adapter is better.

The performance difference between WSI and patch classification in Table \ref{tablecamelyon} for ABMIL \cite{abmil} and DSMIL \cite{dsmil} is due to their use of attention scores, proxies for actual labels. Mean-pooling shows high ROI detection results because it uses instance-level MIL, making it a strong instance classifier but a weak bag classifier. Mean-pooling's approach, a linear layer followed by a mean function, struggles with the unbalanced nature of CAMELYON16's tumor regions (less than 10\%). This imbalance causes normal patches to dominate the final decision, leading to a high false-negative rate in bag classification.

The observed underperformance of HIPT \cite{hipt} on the challenging CAMELYON16 dataset \cite{camelyon16} can be attributed to its dependency on two sequential dimension reduction steps via DINO pre-training \cite{dino}, which introduces substantial error accumulation detrimental to the classifier's final stage. This issue is combined with the model's reliance on limited input data for a data-hungry ViT.

In ROI detection, our model surpasses existing methods by a significant margin, establishing a new SOTA. The inferior performance of DGMIL \cite{dgmil} in this area is linked to its dependency on noisy labels, which are generated through clustering with a preset and potentially inaccurate number of clusters for embedding enhancement.

Furthermore, our findings across MIL datasets show the versatility of our MIL-pooling strategy, affirming its efficacy as a broadly applicable architecture within the MIL framework which can be seen in Table \ref{tablemildatasets}.

For qualitative ROI detection evaluation check Fig. \ref{fig:2}.

\begin{figure*}
  \centering 
  \hspace{5mm}\begin{subfigure}[b]{0.4\textwidth}
    \centering
    \includegraphics[width=\textwidth]{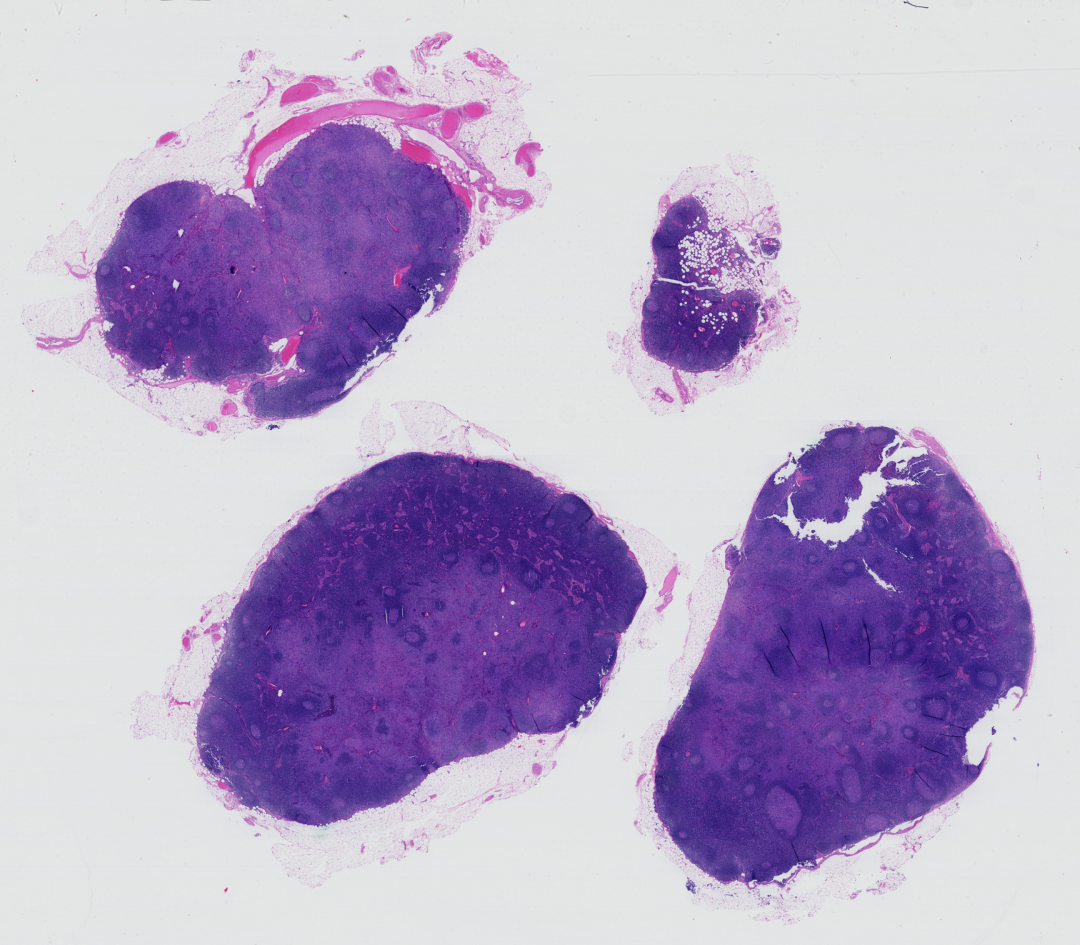}
    \caption{} \label{fig:2a}
  \end{subfigure}
  \begin{subfigure}[b]{0.4\textwidth}
    \centering
    \includegraphics[width=\textwidth]{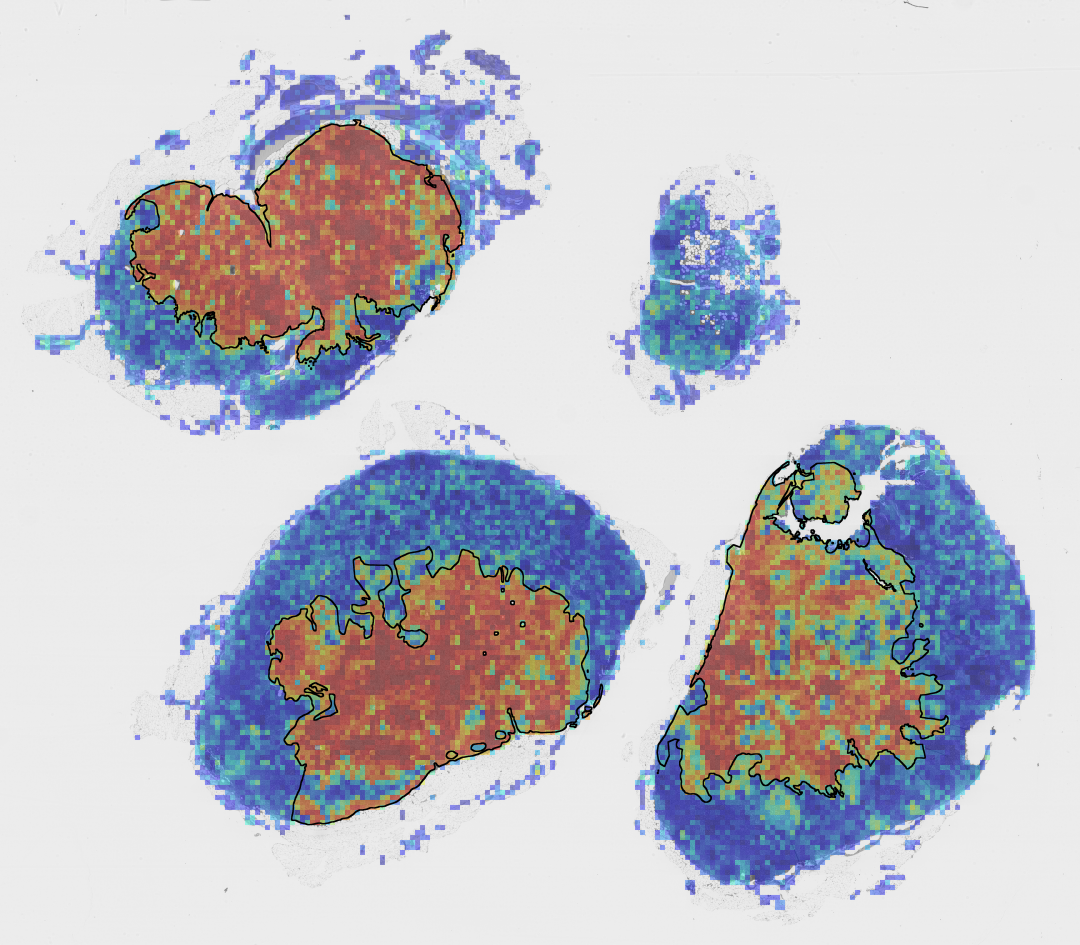}
    \caption{} \label{fig:2b}
  \end{subfigure}
  \begin{subfigure}[b]{0.1\textwidth} 
    \raisebox{3.9mm}{
      \includegraphics[width=0.321\textwidth,height=43.0mm]{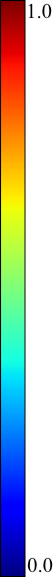}
      }
  \end{subfigure}
  \caption{Qualitative view of ROIs recognized by Snuffy through its Patch Classification. (a) An example WSI from the test set of the CAMELYON16 dataset \cite{camelyon16}. (b) ROIs are identified by Snuffy with black lines delineating the ground truth ROIs.}
  \label{fig:2}
\end{figure*}

\begin{table*}[t]
    \setlength{\tabcolsep}{5pt}
    \centering
    \resizebox{\textwidth}{!}{
    \begin{tabular}{l|ccc|cc|cc|cc}
    \hline
    \multicolumn{1}{c}{Method} &
    \multicolumn{3}{c}{Slide} &
    \multicolumn{2}{c}{Patch} &
    \multicolumn{2}{c}{System} &
    \multicolumn{2}{c}{Resource} \\
    \cline{2-10}
    \multicolumn{1}{c}{} &
    ACC & AUC & \multicolumn{1}{c}{ECE} & AUC & \multicolumn{1}{c}{FROC} & Mem & \multicolumn{1}{c}{GPU Mem} & \(\sum\)Time & \(\sum\)\#Params \\
    \multicolumn{1}{c}{} &  &  & \multicolumn{1}{c}{} & & \multicolumn{1}{c}{} & (GB) & \multicolumn{1}{c}{(GB)} & (Minute) & (Million) \\
    \hline
    Mean-pooling SimCLR Exhaustive & \( 0.614 \) & \( 0.695 \) & \( 0.070 \) & \( \mathbf{0.980} \) & \( 0.597 \) & \( 8.3 \) & \( 2.4\) & \( 806,447 \)& \( 11.56 \)
    \\
    Max-pooling SimCLR Exhaustive & \( 0.829 \) & \( 0.848 \) & \( 0.056 \) & \( 0.790 \) & \( 0.387 \) & \( 8.3 \) & \( 2.4\) & \( 806,447 \) & \( 11.56 \)
    \\
    ABMIL SimCLR Exhaustive \cite{abmil} & \( 0.815 \) & \( 0.762 \) & \( 0.184 \) & \( 0.303 \) & \( 0.182 \) & \( 8.3 \) & \( 1.62\) & \( 806,422 \) & \( 11.62 \)
    \\
    ABMIL-Gated SimCLR Exhaustive \cite{abmil} & \( 0.876 \) & \( 0.847 \) & \( 0.124 \) & \( 0.394 \) & \( 0.245 \) & \( 8.3 \) & \( 1.62\) & \( 806,423 \) & \( 11.69 \)
    \\
    DGMIL \cite{dgmil} & \( 0.801 \) & \( 0.836 \) & \( NA \) & \( 0.904 \) & \( 0.488 \) & \( 76\) & \( 11\) & \( 1,612,920 \) & \( 111.66 \)
    \\
    DSMIL \cite{dsmil} & \( 0.810 \) & \( 0.858 \) & \( 0.051 \) & \( 0.411 \) & \( 0.328 \) & \( 19.2 \) & \( 3.09 \) & \( 806,448 \) & \( 11.64 \)
    \\
    TransMIL \cite{transmil} & \( 0.815 \) & \( 0.843 \) & \( 0.125 \) & \( - \) & \( - \) & \( 1.51\) & \( 7.87\) & \( 8.5 \) & \( 2.7 \)
    \\
    CLAM-SB \cite{clam} & \( 0.820 \) & \( 0.825 \) & \( 0.086 \) & \( - \) & \( - \) & \( 1.16\) & \( 1.97\) & \( 11 \) & \( 0.92 \)
    \\
    CLAM-MB \cite{clam} & \( 0.821 \) & \( 0.851 \) & \( 0.110 \) & \( - \) & \( - \) & \( 1.15\) & \( 1.97\) & \( 11 \) & \( 0.92 \)
    \\
    HIPT \cite{hipt} & \( 0.634 \) & \( 0.641 \) & \( 0.122 \) & \( - \) & \( - \) & \( 0.92 \) & \( 2.40\) & \( 15,384 \) & \( 69.35 \)
    \\
    KAT (w/o KCL) \cite{kat} & \( 0.579 \) & \( 0.576 \) & \( 0.290 \) & \( - \) & \( - \) & \( 2.52\) & \( 16.43\) & \( 706 \) & \( 61.2 \)
    \\
    KAT (w/ KCL) \cite{kat} & \( 0.576 \) & \( 0.551 \) & \( 0.301 \) & \( - \) & \( - \) & \( 30\) & \( 15.22\) & \( 830 \) & \( 66.4 \)
    \\
    DTFD-MIL (MaxS) \cite{dtfd} & \( 0.821 \) & \( 0.829 \) & \( 0.107 \) & \( - \) & \( - \) & \( 29.73\) & \( 1.38 \) & \( 38 \) & \( 9.86 \)
    \\
    DTFD-MIL (MaxMinS) \cite{dtfd} & \( 0.818 \) & \( 0.837 \) & \( 0.101 \) & \( - \) & \( - \) & \( 29.73\) & \( 1.50\) & \( 37 \) & \( 9.86 \)
    \\
    DTFD-MIL (AFS) \cite{dtfd} & \( 0.832 \) & \( 0.855 \) & \( 0.083 \) & \( - \) & \( - \) & \( 29.72\) & \( 1.47\) & \( 36 \) & \( 9.86 \)
    \\\hline
    Snuffy SimCLR Exhaustive & \( \mathbf{0.952} \) & \( 0.970 \) & \( \mathbf{0.057} \) & \( \mathbf{0.980} \) & \( 0.622 \) & \( 8.3 \) & \( 4.8 \) & \( 806,496 \) & \( 14.87 \)
    \\
    Snuffy DINO Adapter & \( 0.876 \) & \( 0.936 \) & \( 0.058 \) & \( 0.911 \) & \( 0.552 \) & \( 6.7 \) & \( 4.8 \) & \( 1,536 \) & \( 23.7 \)
    \\
    Snuffy DINO Full-tune & \( 0.761 \) & \( 0.756 \) & \( 0.195 \) & \( 0.881 \) & \( 0.381 \) & \( 6.7 \) & \( 4.8 \) & \( 1,658 \) & \( 47.5 \)
    \\
    Snuffy DINO Exhaustive & \( 0.948 \) & \( \mathbf{0.987} \) & \( 0.083 \) & \( 0.957 \) & \( \mathbf{0.675} \) & \( 6.7 \) & \( 4.8 \) & \( 15,393 \) & \( 47.5 \)
    \\
    Snuffy MAE Adapter & \( 0.900 \) & \( 0.910 \) & \( 0.078 \) & \( 0.873 \) & \( 0.543 \) & \( 11.7 \) & \( 4.8 \) & \( 1,056 \) & \( 8.1 \)
    \\
    Snuffy MAE Full-tune & \( 0.782 \) & \( 0.754 \) & \( 0.134 \) & \( 0.875 \) & \( 0.363 \) & \( 11.7 \) & \( 4.8 \) & \( 1,216 \) & \( 113.66 \)
    \\\hline
    \end{tabular}
    }
    \caption{Results on The CAMELYON16 dataset \cite{camelyon16}. Exhaustive is for SSL models trained from scratch, Full-tune for models trained from ImangeNet-1K pre-trained weights, and adapter for models trained with our proposed method. Patch is for patch classification and ROI detection. \(\sum\)Time shows the sum of pre-training time and MIL training time. \(\sum\)\#Params shows the sum of pre-trained parameters and MIL-trained parameters, and \( - \) shows the model is incapable of or has not implemented a quantifiable method for ROI detection. Mem and GPU Mem numbers are not the most accurate due to randomness in CUDA and parallel runs but accurate enough to give a good comparison.}
    \label{tablecamelyon}
\end{table*}

\begin{table*}[t]
    \begin{minipage}{0.35\linewidth}
    \setlength{\tabcolsep}{5pt}
    \centering
    \resizebox{\textwidth}{!}{
    \begin{tabular}{l|cc}
    \hline
    \multicolumn{1}{c}{Method} &
    \multicolumn{2}{c}{Slide} \\
    \cline{2-3}
    \multicolumn{1}{c}{} &
    ACC & AUC \\\hline
    DGMIL \cite{dgmil} & \( 0.920 \) & \( 0.970 \)
    \\
    TransMIL \cite{transmil} & \( 0.883 \) & \( 0.949 \)
    \\
    CLAM-SB \cite{clam} & \( 0.875 \) & \( 0.944 \)
    \\
    CLAM-MB \cite{clam} & \( 0.878 \) & \( 0.949 \)
    \\
    HIPT \cite{hipt} & \( NA \) & \( 0.952 \)
    \\
    KAT (w/o KCL) \cite{kat} & \( 0.849 \) & \( 0.965 \)
    \\
    KAT (w/ KCL) \cite{kat} & \( 0.859 \) & \( 0.971 \)
    \\
    DTFD-MIL (MaxS) \cite{dtfd} & \( 0.855 \) & \( 0.910 \)
    \\
    DTFD-MIL (MaxMinS) \cite{dtfd} & \( 0.890 \) & \( 0.938 \)
    \\
    DTFD-MIL (AFS) \cite{dtfd} & \( 0.898 \) & \( 0.946 \)
    \\\hline
    Snuffy SimCLR Exhaustive & \( \mathbf{0.947} \) & \( \mathbf{0.972} \)
    \\\hline
    \end{tabular}
    }
    \caption{Results on The TCGA dataset \cite{tcga}.}
    \label{tabletcga}
    \end{minipage}
\hfill
    \begin{minipage}{0.63\linewidth}
    \centering
    \resizebox{\textwidth}{!}{
    \begin{tabular}{l|cc|cc|cc}
    \hline
    \multicolumn{1}{c}{Method} &
    \multicolumn{2}{c}{MUSK1} &
    \multicolumn{2}{c}{MUSK2} &
    \multicolumn{2}{c}{ELEPHANT} \\
    \cline{2-7} 
    \multicolumn{1}{c}{} &
    ACC & \multicolumn{1}{c}{AUC} & ACC & \multicolumn{1}{c}{AUC} & ACC & \multicolumn{1}{c}{AUC}
    \\\hline
    Max-pooling & \( 0.728\) & \( 0.799\) & \( 0.676\) & \( 0.756\) & \( 0.764\) & \( 0.861\)
    \\
    Mean-pooling & \( 0.800\) & \( 0.869\) & \( 0.710\) & \( 0.855\) & \( 0.830\) & \( 0.920\)
    \\
    ABMIL \cite{abmil} & \( 0.826\) & \( 0.824\) & \( \mathbf{0.812}\) & \( 0.816\) & \( 0.849\) & \( 0.84\)
    \\
    ABMIL-Gated \cite{abmil} & \( 0.831\) & \( 0.837\) & \( 0.780\) & \( 0.784\) & \( 0.789\) & \( 0.844 \) 
    \\
    DSMIL \cite{dsmil} & \( 0.786 \) & \( 0.852 \) & \( 0.706 \) & \( 0.813\) & \( 0.811\)  & \( 0.915 \)
    \\\hline
    Snuffy & \( \mathbf{0.961} \) & \( \mathbf{0.989 } \) & \( 0.789 \) & \( \mathbf{0.985} \) & \( \mathbf{0.923} \) & \( \mathbf{0.967} \)
    \\\hline
    \end{tabular}
    }
    \caption{Results on MUSK1, MUSK2 \cite{muskmildatasets}, ELEPHANT \cite{animalmildatasets}.}
    \label{tablemildatasets}
    \end{minipage}
    \end{table*} 

\section{Ablation Study}
\label{ablation}

In this section, we explore the impact and efficacy of key components of Snuffy on the classification and ROI detection performance on the CAMELYON16 dataset \cite{camelyon16}. These experiments are done with Snuffy simCLR Exhaustive.

\textbf{Effect of Depth:} The architecture's depth, evaluated at 1, 2, 4, and 5 layers, is analyzed to determine its influence on the model's performance as a universal approximator. As depicted in Fig. \ref{fig:depth}, an increase in depth, generally but slightly correlates with enhanced performance and complies with our theoretical insights.

\textbf{Effect of Number of Random Global Attentions:} The examination of the number of Random Global Attentions, at 1, 300, and 700, demonstrates an increase in performance with higher numbers, as illustrated in Fig. \ref{fig:randomglobal}. This and the observation of diminishing returns beyond 1000 aligns with the theoretical insights.

For more ablation studies check Supplementary Material \ref{additional_ablation}.

\begin{figure}[!htbp]
    \centering
    \begin{minipage}{0.5\textwidth}
        \centering
        \includegraphics[width=1\textwidth]{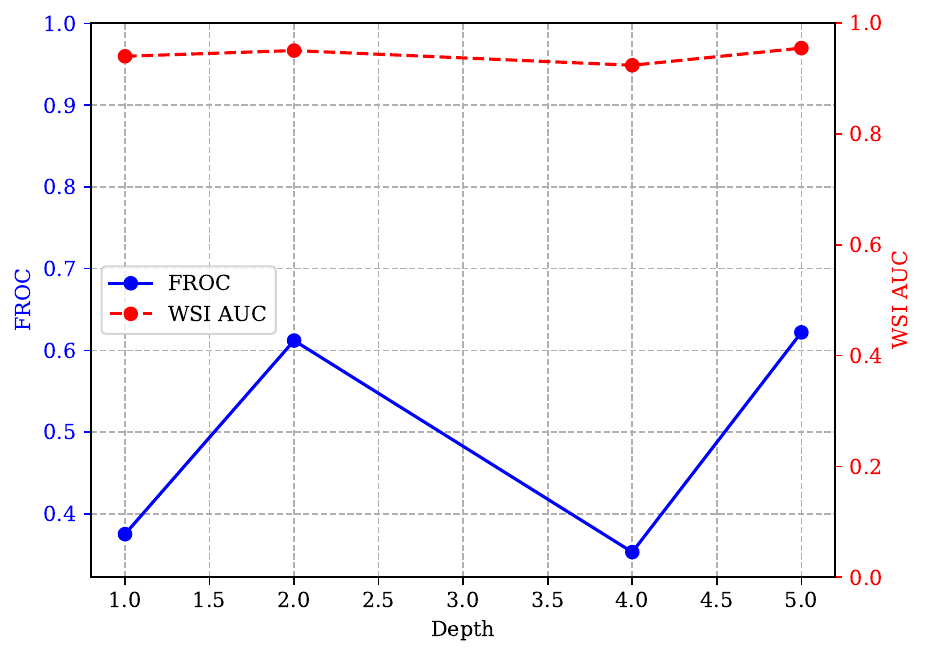} 
        \caption{Ablation on depth.}
        \label{fig:depth}
    \end{minipage}\hfill
    \begin{minipage}{0.5\textwidth}
        \centering
        \includegraphics[width=1\textwidth]{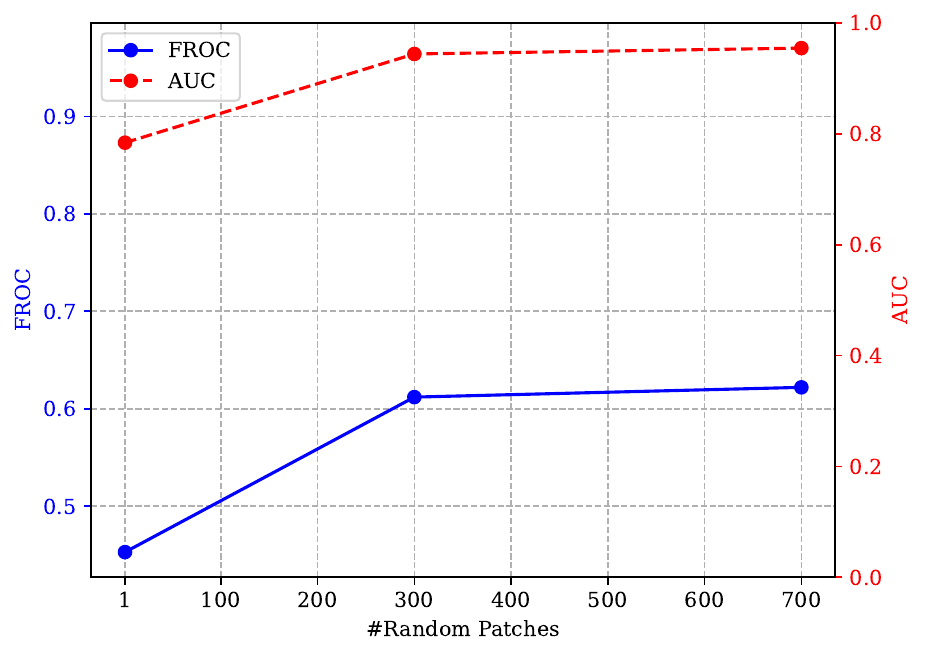} 
        \caption{Random Global Attentions.}
        \label{fig:randomglobal}
    \end{minipage}
\end{figure}

\section{Conclusion and Discussion}
We introduced a novel WSI classification framework using PEFT for data efficiency in SSL and a sparse transformer inspired by pathologists. This approach ensures global approximation with high probability and provides a tighter bound for the number of layers in sparse transformers for MIL-pooling, achieving excellent results. However, achieving SOTA results still requires long and exhaustive SSL training. Our Theoretical guarantees need a high number of layers, increasing memory needs (though one layer performs very well empirically). Future work could explore PEFT methods tailored to pathology and develop less resource-intensive MIL-pooling methods. Another avenue is closing the theoretical and practical gap in understanding sparse transformers.


\section*{Acknowledgements}
\label{ack}
We thank Mohammad Mosayyebi, Mehrab Moradzadeh, Mohammad Hosein Movasaghinia, Mohammad Azizmalayeri, Hossein Mirzaei, Mohammad Mozafari, Soroush Vafaei Tabar, Mohammad Hassan Alikhani, and Hosein Hasani.

%
%
\bibliographystyle{splncs04}
\bibliography{main}

\clearpage
\pagenumbering{arabic}
\setcounter{page}{1}

\setcounter{page}{1}

\title{Snuffy: Efficient Whole Slide Image Classifier (Supplementary Materials)} 

\titlerunning{Snuffy Supplementary}

\author{Hossein Jafarinia\inst{1}\orcidlink{0009-0003-4172-8372} \and
Alireza Alipanah\inst{1}\orcidlink{0009-0000-1292-9296} \and 
Danial Hamdi\inst{2}\orcidlink{0009-0005-1419-3338} \and
Saeed Razavi\inst{1}\orcidlink{0009-0009-0760-0758} \and 
Nahal Mirzaie\inst{1}\orcidlink{0009-0003-6954-7151} \and 
Mohammad Hossein Rohban\inst{1}\orcidlink{0000-0001-6589-850X} \thanks{Corresponding author.}}

\authorrunning{H. Jafarinia et al.}

\institute{Sharif University of Technology, Tehran, Iran \\
\email{\{jafarinia, alireza.alipanah46, saeed.razavi, nahal.mirzaie, rohban\}@sharif.edu}\\
\url{https://www.sharif.edu}
\and Amirkabir University of Technology (Tehran Polytechnic), Tehran, Iran \\
\email{danial.hamdi@outlook.com}\\
\url{https://aut.ac.ir}
}

\maketitle

\section{Theory}
\begin{proposition}\label{propositionproof}
    Every graph $G(E,V)$ with $E$ set of edges and $V$ as set of nodes, with $\rvert V\lvert \geq3$ and $\alpha(G) \leq \chi(G)$ has a Hamiltonian cycle. Where $\alpha(G)$ is the maximum independent set, and $\chi(G)$ is the chromatic number of $G$.
\end{proposition}
\begin{proof}
A detailed proof can be found in Chapter 10, Proposition 10.1.2 of the book by \cite{graphtheory}.
\end{proof}

\begin{lemma}\label{lemmaproof}
    For $G_S$, the graph representing Snuffy sparsity patterns, we guarantee that there exist a Hamiltonian cycle if
    \begin{equation}
        \rvert \bigcup\limits_{l\in[L]} \Lambda^l\lvert \geq \frac{n-1}{2}
    \end{equation}
    Where $\Lambda^l$ is set of patches selected in layer $l$ which consists of $\Lambda_{top}$ and $\Lambda_{r}^l$.
\end{lemma}
\begin{proof}
    Let's define $\Lambda = \bigcup_{l\in[L]} \Lambda^l$ and $N' = N\backslash \Lambda$. At each layer, the selected patches represent global attentions, thus we can consider $\Lambda$ as a clique in $G_S$, with every node in $N'$ attending to them. Moreover, in each layer, at most $\lambda_r$ nodes from $N'$ are added to $\Lambda$ (see Fig.~\ref{fig:lemma_bipar}).
    \begin{figure}
        \centering
        \includegraphics[width=0.4\textwidth]{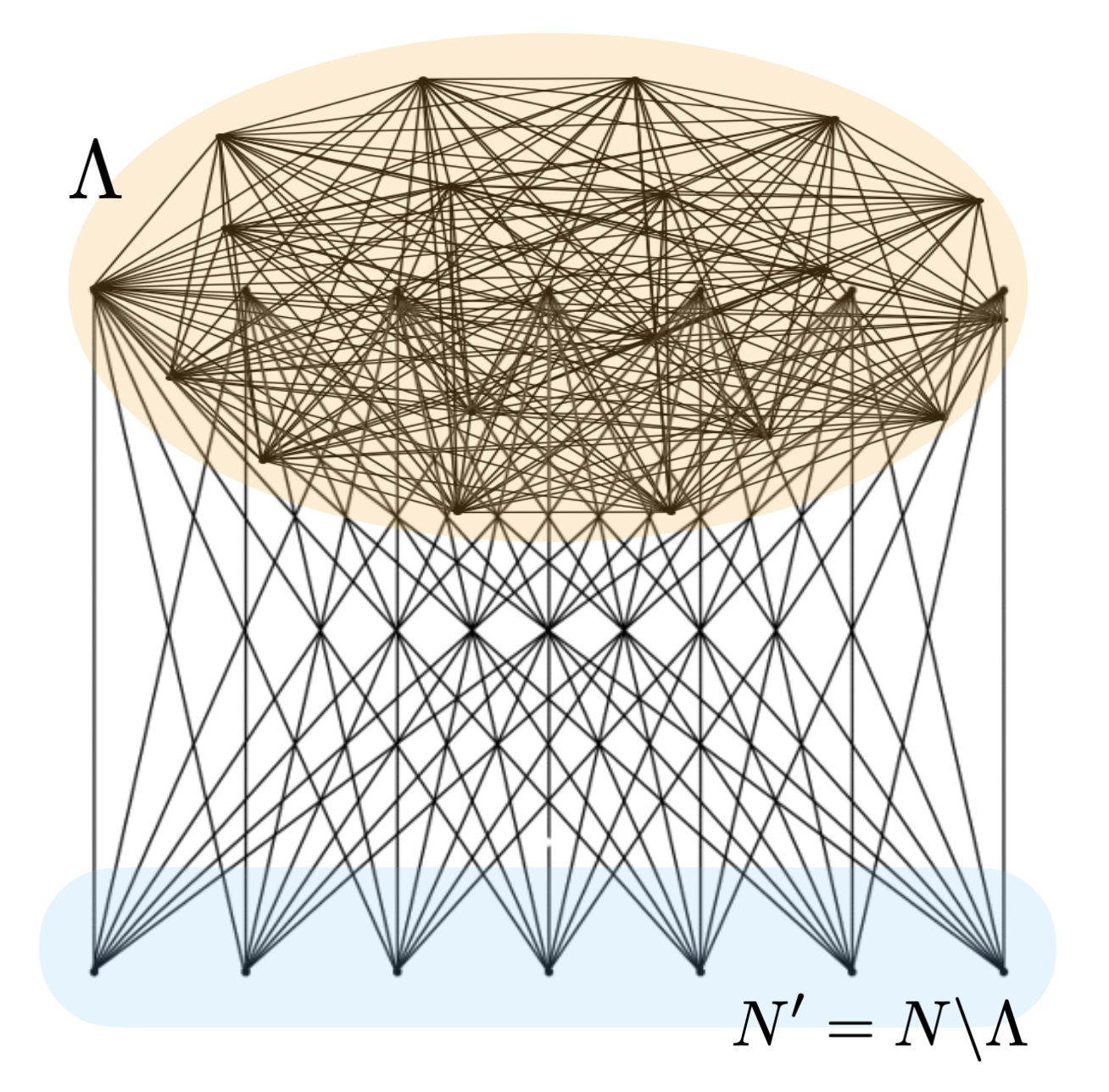}
        \caption{Graphical representation of the Snuffy sparsity patterns graph $G_S$ up to layer $l$. $\Lambda$ represents the set of patches that have been observed, while $N'$ denotes the set of patches that have not been covered. Please note that self-loops for all nodes are omitted for simplicity.}
        \label{fig:lemma_bipar}
    \end{figure}
    It is evident that the chromatic number of $G_S$, denoting the minimum number of colors required to ensure that no two neighboring nodes share the same color, is equal to $|\Lambda| + 1$. This arises from the necessity of using $|\Lambda|$ colors to color the nodes in $\Lambda$ and an additional color to color the rest of the nodes in $N'$. On the other hand, the maximum independent set of patches in $G_S$ is equal to $N'$. By leveraging Proposition~\ref{propositionproof} and the fact that $|N' \cup \Lambda| = n$, we establish the lemma. $\blacksquare$
\end{proof}

\begin{theorem}\label{maintheoremproof}
    If $\lambda_r = \mathcal{O}(n)$ where n is number of patches, and $\lambda_r = \rvert \Lambda_r \lvert$, then the number of layers $L$ needed to prove that the Snuffy sparsity pattern (defined in ~\ref{sparsitypattern}) is universal approximator sparsity patterns is sharply concentrated around $\frac{n\log{2}}{\lambda_r}$. More precisely, we have:
    \begin{equation}
        \lim_{n \to \infty} \mathbb{P}( \rvert L - \frac{n\log{2}}{\lambda_r} \lvert > \frac{c\sqrt{n}}{\lambda_r}) = 1 - \Phi(c)
    \end{equation}
\end{theorem}
\begin{proof}
    In Section \ref{univ_appr}, we illustrate that for our proposed Snuffy sparsity patterns to be universal approximator sparsity patterns, it is sufficient to satisfy three conditions as outlined in Theorem \ref{three}. We also establish that Conditions 1 and 3 are straightforward, and it is only needed to confirm the satisfaction of condition 3. In this regard, in Lemma \ref{lemma}, we discuss that observing $\lfloor{\frac{n}{2}}\rfloor$ patches is sufficient  to ensure that the graph associated with the sparsity patterns is Hamiltonian, thereby satisfying Condition 3. 

    Observing at least $\lfloor{\frac{n}{2}}\rfloor$ patches within our context equates to a generalized coupon collector problem. In a classic coupon collector scenario with $n$ distinct types of coupons, we randomly select singleton coupons with uniform probability, aiming to eventually collect all types of coupons. In our scenario, we can analogously envision selecting $\lambda_r$ random patches uniformly, with the objective of acquiring at least $\lfloor{\frac{n}{2}}\rfloor$ of all patches. Notably, for $\lambda_{top} = \mathcal{O}(\lambda_r)$ we can presume that the $\lambda_{top}$ patches are initially selected, and their presence in subsequent layers does not change the threshold.

    Let $T_{n,\frac{n}{2}}=L$ denote a variable representing the number of steps or, in our case, the number of layers, at which for the first time at least $\lfloor{\frac{n}{2}}\rfloor$ patches have been observed. Our goal is to show that $T_{n,\frac{n}{2}}$ is concentrated at $\frac{n\log{2}}{\lambda_r}$.
    To achieve this, \textbf{1)} we begin by revisiting \cite{baum} theorem for concentration $T_{n,m}$, which pertains to the number of steps required to pick $m$ coupons out of $n$ coupons in the classic coupon collector problem where only one coupon is chosen at each step.  \textbf{2)} We then extend this bound to the $\lambda_r$-uniform coupon collector problem by coupling the $\lambda_r$-uniform coupon sequence with the sequence obtained by the classic singleton coupon collector.
\end{proof} 

\subsection*{Revisiting Baum's Theorem: Concentration of $T_{n,\frac{n}{2}}$ }
    \begin{theorem}\label{baum}
        For $T_{n,m}$ representing the number of steps required to pick $m$ coupons out of $n$ coupons in the classic coupon collector problem, where $m$ approaches infinity alongside $n$, but at a rate slow enough to allow the sequence $\frac{n - m}{\sqrt{n}}$ to also tend to infinity, we have
        \begin{equation}
            \lim_{n \to \infty} \mathbb{P}(\frac{T_{n,m} - \mathbb{E}T_{n,m}}{\sigma_n} \leq c) = \Phi(c)
        \end{equation}
        where for $\frac{m}{n} \rightarrow d$, $\sigma_n^2 \sim n^{\frac{1-d+d \log d}{d}}$ \cite{baum}.
    \end{theorem}
    This theorem shows that $T_{n,m}$ concentrates around its mean with maximum bound of $c\sigma_n$, with a probability of at least $\Phi(c)$, where $\Phi(c)$ represents the cumulative distribution function ($CDF$) of the standard normal distribution $\mathcal{N}(0, 1)$. Substituting $m$ in Theorem \ref{baum} with $\frac{n}{2}$ in our case, we derive the following lemma:
    \begin{lemma}\label{lem:half}
        For $T_{n,\frac{n}{2}}$ representing the number of steps required to pick half of the coupons out of $n$ coupons in the classic coupon collector problem, we have
        \begin{equation}
        \begin{aligned}            
            \lim_{n \to \infty} \mathbb{P}(\frac{T_{n,\frac{n}{2}} - \mathbb{E}T_{n,\frac{n}{2}}}{\sqrt{n^{1-\log 2}}} \leq c) = \Phi(c) \\
        \blacksquare
        \end{aligned}
        \end{equation}
    \end{lemma}

\subsection*{$\lambda_r$-uniform coupon collector}
To generalize the bound in Lemma~\ref{lem:half} for the $\lambda_r$-uniform coupon collector problem, we follow the proof provided in \cite{gcc} for coupling any $\lambda_r$-uniform coupon sequence to the sequence of coupons received by the classic singleton coupon collector. 

Let $Y_i \sim \mathbf{Y}$ denote the classical random covering variable for $\Lambda$. We define the smallest set $\{a_0, a_1, \ldots\}$ such that $|\bigcup_{i\in[a_i, a_{i+1}]} Y_i| = \lambda_r$. Intuitively, if we start from the beginning of the sequence of $\mathbf{Y}$, $a_1$ is the first time when we encounter $\lambda_r$ distinct coupons. In a greedy manner, $a_2$ is then the first time after $a_1$ that we again collect $\lambda_r$ distinct coupons, and so forth. We can then declare $X_i = \bigcup_{i\in[a_i, a_{i+1}]} Y_i$ to denote the random covering variable for $\Lambda$, obtained by uniformly selecting a $\lambda_r$-set from $\Lambda$. It is clear that both $\mathbf{Y}$ and $\mathbf{X}$ are independent random variables and they are related by the following equation:
\begin{equation}\label{eq:xy}
    T_{n,\frac{n}{2}}(\mathbf{X}) = \bigcup\limits_{j=1}^t X_j = \bigcup\limits_{i=1}^{a_t} Y_i = T_{n,\frac{n}{2}}(\mathbf{Y})
\end{equation}
It is easy to infer from equation~\ref{eq:xy} that: 
\begin{equation}\label{eq:xy_2}
   T_{n,\frac{n}{2}}(\mathbf{X}) \leq t \iff T_{n,\frac{n}{2}}(\mathbf{Y}) \leq a_t
\end{equation}

We can define $l_i = a_{i+1}-a_i$ and $S_m = \sum_{i\in[m]} l_i$ respectively. Then using the equation~\ref{eq:xy_2} we can find a lower and upper bound for $T_{n,\frac{n}{2}}(\mathbf{Y})$.

\begin{equation}
\begin{aligned}
   \sum_{i\in[T_{n,\frac{n}{2}}(\mathbf{X})-1]} l_i = a_{T_{n,\frac{n}{2}}(\mathbf{X}) -1 } < T_{n,\frac{n}{2}}(\mathbf{Y}) \leq
   a_{T_{n,\frac{n}{2}}(\mathbf{X})} = 
   \sum_{i\in[T_{n,\frac{n}{2}}(\mathbf{X})]} l_i 
\end{aligned}
\end{equation}

\begin{equation}
\begin{aligned}
   S_{T_{n,\frac{n}{2}}(\mathbf{X}) -1}  
   < T_{n,\frac{n}{2}}(\mathbf{Y}) \leq 
   S_{T_{n,\frac{n}{2}}(\mathbf{X})}
\end{aligned}\label{eq:xy_3}
\end{equation}

\begin{lemma}\label{lem:gcc}
    If $\lambda_r = \mathcal{O}(n)$ then for all $m > n\lambda_r$ the following inequality holds:
    \begin{equation}
        \mathbb{P}(
            |S_m - \mathbb{E}S_m| > \sqrt{m\lambda_r}
        ) < 4\exp^{-\frac{n}{\lambda_r}}
    \end{equation}
\end{lemma}
\begin{proof}
    Lamma 2 in \cite{gcc}.
\end{proof}
Utilizing equation~\ref{eq:xy_3} and lemma~\ref{lem:gcc}, and substituting $m=T_{n,\frac{n}{2}}(\mathbf{X})$ we can assert the following equation is true with high probability (w.h.p):

\begin{equation}
\begin{aligned}
   \lambda_r(T_{n,\frac{n}{2}}(\mathbf{X}) - 1) - \sqrt{\lambda_r(T_{n,\frac{n}{2}}(\mathbf{X}) - 1)}
   <
    T_{n,\frac{n}{2}}(\mathbf{Y})\\
    T_{n,\frac{n}{2}}(\mathbf{Y})
    \leq
    \lambda_rT_{n,\frac{n}{2}}(\mathbf{X})- \sqrt{\lambda_rT_{n,\frac{n}{2}}(\mathbf{X})}
    \\
    \blacksquare
\end{aligned}
\end{equation}
\\

So far we show that we can $T_{n,\frac{n}{2}}(\mathbf{Y})$ can be bounded by variables of $\lambda_r$ and $T_{n,\frac{n}{2}}(\mathbf{X})$. Also, by lemma~\ref{lem:half} we know that $|T_{n,\frac{n}{2}}(\mathbf{Y}) - \mathbb{E}T_{n,\frac{n}{2}}(\mathbf{Y})| \leq c\sqrt{n^{1-\log2}}$ holds with probability of $\Phi(c)$. Also it is well-known equation that $\mathbb{E}T_{n,\frac{n}{2}}(\mathbf{Y}) = n\log2$. Meaning that the expectation of collecting at least half of the coupons is equal to $n\log 2$. Further, we can argue that $T_{n,\frac{n}{2}}(\mathbf{X}) < \frac{2n\log2}{\lambda_r}$. This can be deduce from lemma~\ref{lem:half}. With all the pieces of the puzzle in place and using the triangle inequality, we have:

\begin{equation}\label{eq:fin}
\begin{aligned}
    \bigl\vert T_{n,\frac{n}{2}}(\mathbf{X}) - \frac{n \log 2}{\lambda_r} 
    \bigr\vert \leq
    \bigl\vert \frac{T_{n,\frac{n}{2}}(\mathbf{Y})}{\lambda_r} - \frac{n\log2}{\lambda_r}
    \bigr\vert +
    \bigl\vert T_{n,\frac{n}{2}}(\mathbf{X}) - \frac{T_{n,\frac{n}{2}}(\mathbf{Y})}{\lambda_r} \bigr\vert \\
    \leq 
    \frac{c\sqrt{n^{1-\log2}}}{\lambda_r} + \frac{\sqrt{T_{n,\frac{n}{2}}(\mathbf{X})}}{\lambda_r} + 1 \\
    \leq
    \frac{c\sqrt{n^{1-\log 2}}}{\lambda_r} +
    \frac{\sqrt{2n\log 2}}{\lambda_r \sqrt{\lambda_r}} + 1\\
    \leq
    \frac{(c+o(1))\sqrt{n}}{\lambda_r}
\end{aligned}
\end{equation}

Equation~\ref{eq:fin} indicates that $\bigl\vert T_{n,\frac{n}{2}}(\mathbf{X}) - \frac{n \log 2}{\lambda_r} 
    \bigr\vert \leq \frac{c\sqrt{n}}{\lambda_r}$ with probability at least $\Phi(c)$. By considering the complement of this event, we have proven Theorem \ref{maintheoremproof}. $\blacksquare$

\section{Datasets}
\label{supdatasets}

\textbf{CAMELYON16 dataset} is a widely utilized and publicly available dataset for WSI classification in breast cancer, encompasses 270 training and 129 testing WSIs, delineated into tumor and normal classes. It also offers pixel-level annotations for all its WSIs. Notably, the tumor regions within the positively classified slides of this dataset are exteremely small and imbalanced, constituting less than 10\% of the area \cite{camelyon16}.

\textbf{TCGA Lung Cancer} or The Cancer Genome Atlas Lung Cancer dataset comprises a selection from the vast TCGA database, specifically contianing two subtypes of lung cancer: Lung Adenocarcinoma (LUAD) and Lung Squamous Cell Carcinoma (LUSC). Following the exclusion of low-quality WSIs, the dataset is constituted of 1042 WSIs (530 LUAD and 512 LUSC). In this dataset, the tumor regions within the WSIs account for more than 80\% of them. A critical fact about this dataset is that most patients have multiple WSIs \cite{tcga}.

\textbf{MIL Datasets} or Multiple Instance Learning datasets, are classical datasets designed for the assessment of MIL frameworks. These include Musk1 and Musk2, which are utilized to ascertain whether a drug molecule will bind to a specific target protein. Each molecule is represented through various conformations, with a molecule deemed positive if at least one conformation binds effectively, although the specific conforming binder remains unidentified. The remaining dataset Elephant consists of 200 bags—100 positive and 100 negative. These bags are composed of instances derived from segmented image embeddings. A bag is classified as positive if it contains at least one instance featuring the relevant elephant; otherwise, it is deemed negative. However, the precise labels for individual instances are unknown \cite{muskmildatasets, animalmildatasets}. 

\section{Additional Experimental Setup}
\label{exp_setup_supp}

The patching process of the CAMELYON16 dataset \cite{camelyon16} yielded approximately 3.5 million patches, averaging around 9,000 patches for each WSI. Patches overlapping with annotated tumor regions were classified as tumor patches, whereas the remainder were categorized as normal. We adhered to the official CAMELYON16 \cite{camelyon16} training and testing split, and \(20\%\) of training WSIs as a validation set were selected randomly. Each model was independently executed five times from scratch, and we report the average for each performance metric.

The patching process of the TCGA Lung Cancer dataset \cite{tcga} led to the generation of around 12.5 million patches, with an average of approximately \(12,000\) patches per WSI. The dataset was divided into training, validation, and test sets, constituting roughly \(60\%\), \(15\%\), and \(25\%\) of the total, respectively, ensuring no patient's multiple WSIs were distributed across different sets. Again five independent executions from scratch, with the mean of each metric reported.

For the MIL datasets \cite{muskmildatasets, animalmildatasets}, we implemented a 10-fold cross-validation procedure, conducting five runs per fold, and reported the mean and standard deviation for each metric.

\section{Implementation Details}
\label{imp_details}

For SimCLR Exhaustive, we used pretrained ResNet-18 encoders from \cite{dsmil}, trained separately on the CAMELYON16 and TCGA Lung Cancer datasets \cite{camelyon16, tcga}.

For DINO Exhaustive, we followed \cite{dino} and trained a ViT-S16 from scratch on a domain-specific dataset containing all patches from each training WSI with the default hyperparameters from their official GitHub repository.

For MAE Adapter, we fine-tuned ImageNet-1K pretrained ViT-B/16 plus Adapter on a domain-specific dataset containing 200 random patches from each training WSI through our few-shot learning approach with base learning rate of 0.001 for 400 epochs with 40 warmup epochs.

For DINO Adapter, we fine-tuned ImageNet-1K pretrained ViT-S/16 plus Adapter on a domain-specific dataset containing 50 random patches from each training WSI through our few-shot learning approach with learning rate of 0.0005 and start weight decay of 0.04 and final weight deacy of 0.4 for 100 epochs.

For Snuffy SimCLR Exhaustive, we employed cross-entropy loss, set the Number of Class-Related Global Attentions to 200, the Number of Random Attentions to 700, the Number of Heads to 2, the Number of Layers to 5, the optimizer to AdamW \cite{adamw}, weight decay to 0.05, betas to 0.5 and 0.9, learning rate to 0.0002, and trained for 200 epochs.

For Snuffy DINO Exhaustive, we employed cross-entropy loss, set the Number of Class-Related Global Attentions to 200, the Number of Random Attentions to 700, the Number of Heads to 4, the optimizer to AdamW \cite{adamw}, weight decay to 0.005, betas to 0.9 and 0.999, learning rate to 0.002, and trained for 200 epochs.

For Snuffy DINO Adapter, we employed cross-entropy loss, set the Number of Class-Related Global Attentions to 250, the Number of Random Attentions to 250, the Number of Heads to 4, the optimizer to AdamW \cite{adamw}, weight decay to 0.05, betas to 0.9 and 0.999, learning rate to 0.02, and trained for 200 epochs.

For Snuffy MAE Adapter, we employed cross-entropy loss, set the Number of Class-Related Global Attentions to 250, the Number of Random Attentions to 250, the Number of Heads to 4, the optimizer to AdamW \cite{adamw}, weight decay to 0.005, betas to 0.9 and 0.999, learning rate to 0.02, and trained for 200 epochs.

All MIL-training procedures were conducted only using bag-level labels, incorporating early stopping based on validation set performance to prevent overfitting. These experiments were facilitated using PyTorch (version 2.1) \cite{pytorch} and scikit-learn, running on a system equipped with an RTX 3090.

\section{Calibration Metric ECE}
\label{calibration}

Let \(B_m\) denote the subset of indices for samples whose prediction confidence lies within the interval \(I_m = \left(\frac{m-1}{M}, \frac{m}{M}\right]\). The accuracy for the bin \(B_m\) is:

\begin{equation}
    \text{acc}(B_m) = \frac{1}{|B_m|} \sum_{i \in B_m} 1(\hat{y}_i = y_i),
\end{equation}

where \(\hat{y}_i\) represents the predicted label, and \(y_i\) denotes the actual label for the \(i^{th}\) sample. The mean confidence level for bin \(B_m\) is given by:

\begin{equation}
    \text{conf}(B_m) = \frac{1}{|B_m|} \sum_{i \in B_m} \hat{p}_i,
\end{equation}

with \(\hat{p}_i\) indicating the confidence associated with the \(i^{th}\) sample.

The Expected Calibration Error (ECE) is empirically determined as:

\begin{equation}
    ECE = \sum_{m=1}^{M} \frac{|B_m|}{n} \left| \text{acc}(B_m) - \text{conf}(B_m) \right|,
\end{equation}

For a perfectly calibrated model (\(\text{acc}(B_m) = \text{conf}(B_m)\) for all \(m \in \{1, \ldots, M\}\), yielding ECE of \(0\) \cite{calibration}.

\section{Additional Ablation}
\label{additional_ablation}

We further provide ablation on the number of Class-related Global Attentions in Fig. \ref{fig:class_related_abl} and \ref{fig:shot_abl} and the number of shots for continual training of the MAE \cite{mae} using Adapter. The latter is done with Snuffy MAE Adapter. They also generally show that bigger numbers give better results with a plateau in the former and no apparent plateau in the latter.

\begin{figure}
    \centering
    \begin{minipage}{0.49\textwidth}
        \centering
        \includegraphics[width=1\textwidth]{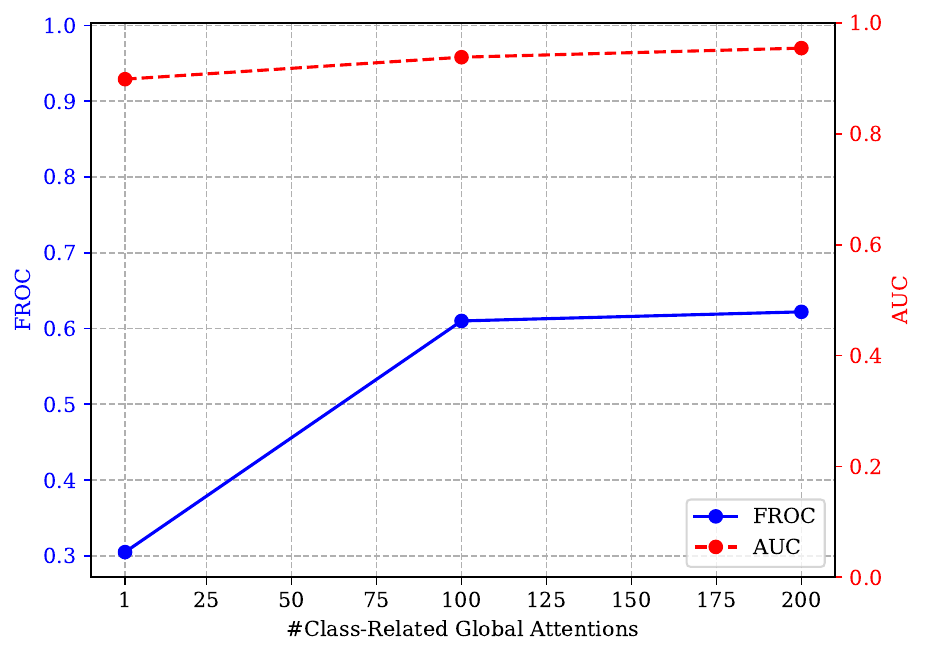} 
        \caption{Ablation on number of Class-Related Global Attentions.}
        \label{fig:class_related_abl}
    \end{minipage}\hfill
    \begin{minipage}{0.49\textwidth}
        \centering
        \includegraphics[width=1\textwidth]{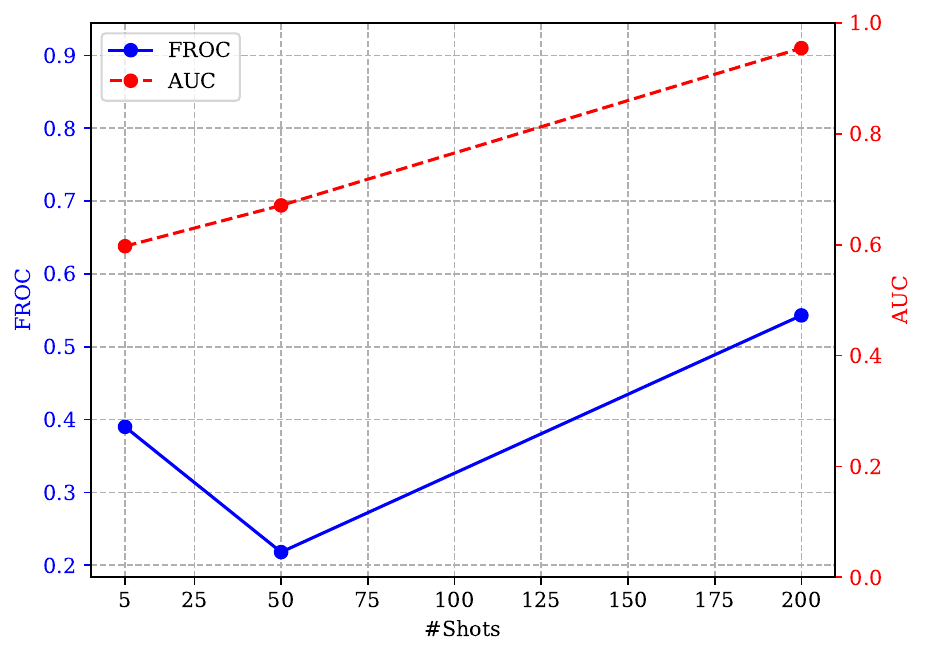} 
        \caption{Ablation on number of shots i.e. number of patches per WSI for MAE \cite{mae} continual training with Adapter.}
        \label{fig:shot_abl}
    \end{minipage}
\end{figure}

\section{Evaluation on Natural Images}
\label{eval_natural}

We further validated our approach by training our model, Snuffy, on classification of natural images to assess its performance. We compared it to the performance of a ViT-S/16 as described \cite{simple_vit}, in MIL setting without positional encoding, which has been shown by \cite{mocov3} to only slightly detract the performance.

For our tests, we configured Snuffy to have a similar structure to ViT-S/16, particularly only except in the self-attention mechanisms. We set \( \lambda_{top} = 2\) and \( \lambda_{r} = 2\), referring to this configuration as Snuffy-S/16. The CIFAR-100 dataset was used as our training data. Both models were trained from scratch for 100 epochs using with image size of \(224\), weight decay of \(0.05\), learning rate of \(0.001\), and the use of the AdamW optimizer \cite{adamw}, with settings optimized for ViT-S/16. Augmentations were applied identically to both models as per established following \cite{simple_vit}.

Results illustrated in Table \ref{vit_vs_snuffy} indicate that Snuffy's approach to MIL pooling is competitive with ViT when it comes to processing a minimal number of patches, underscoring Snuffy's capability as a universal approximator similar to ViT. We suggest that Snuffy could potentially perform even better with optimized hyperparameters. However, it is important to note that the accuracies achieved by both models were not exceptionally high, attributed to the limited data training regimen, which is not ideal for Transformer models.

\begin{table*}[t]
    \setlength{\tabcolsep}{4pt}
    \centering
    \begin{tabular}{*{3}c}
    \\\hline
    Model & ACC@1 & ACC@5 \\\hline
    ViT-S/16 & 0.535 & 0.803 \\
    Snuffy-S/16 & 0.429 & 0.732 \\\hline
    \end{tabular}
    \caption{Performance of ViT-S/16 and Snuffy-S/16 on CIFAR-100 trained from scratch.}
    \label{vit_vs_snuffy}
\end{table*}

\section{t-SNE of MAE Embeddings}
\label{tsne_mae}

We present the t-SNE visualizations for the learned embeddings from ImageNet-1k pre-trained MAE as depicted in Fig. \ref{fig:tsne_mae_a}, continual pre-training with full-tuning in Fig. \ref{fig:tsne_mae_b}, and continual pre-training using the Adapter in Fig. \ref{fig:tsne_mae_c}. From these visualizations, it is evident that ImageNet-1k pre-trained embeddings hold promise. Nonetheless, full-tuning outperforms this baseline, with the Adapter method yielding the most superior results. This underscores the advantage of preserving initial weights while concurrently adjusting to the domain-specific dataset, thereby establishing a more efficacious fine-tuning strategy in Pathology.

\begin{figure*}
  \begin{subfigure}[b]{0.32\textwidth}
    \centering
    \includegraphics[width=\textwidth]{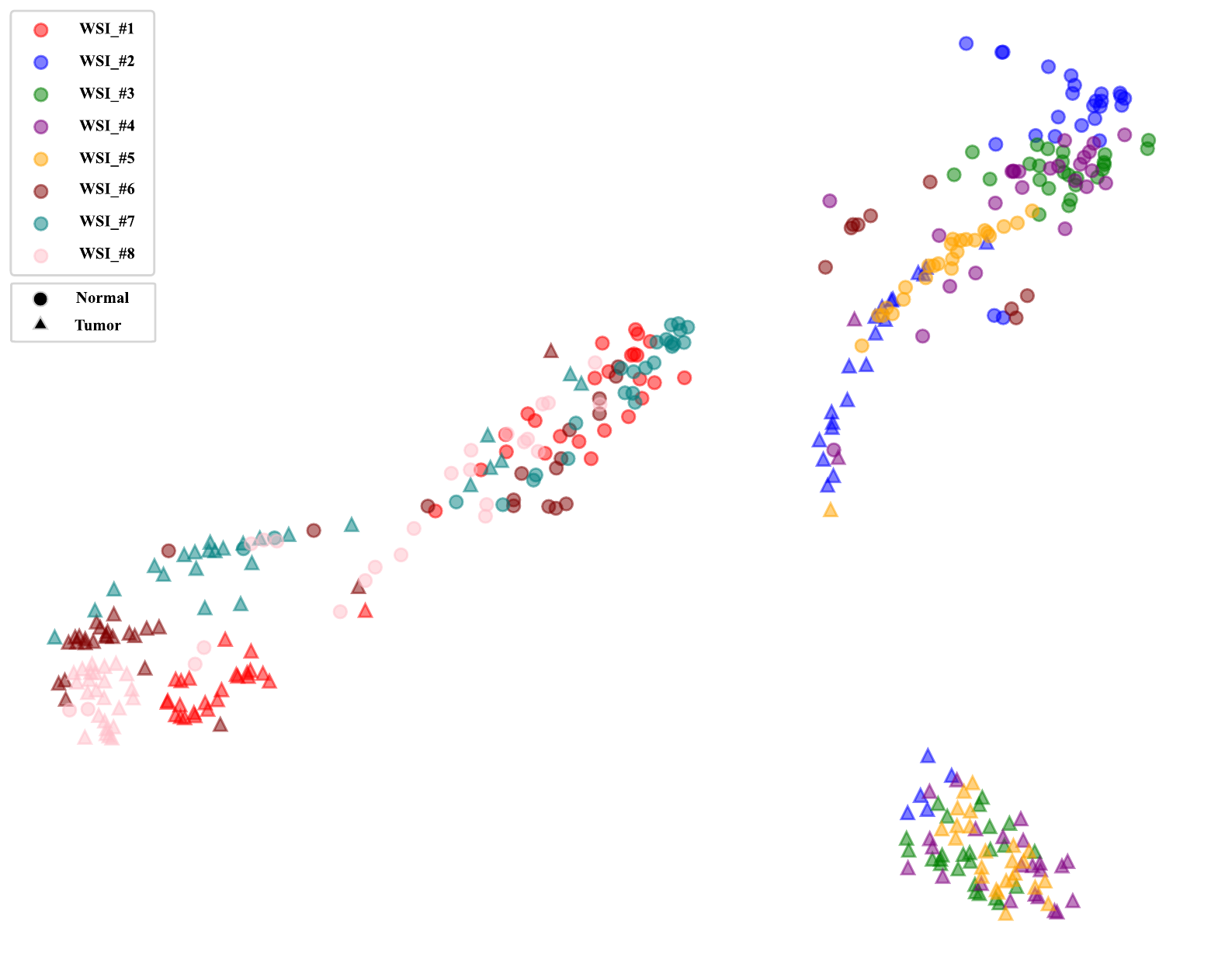}
    \caption{} \label{fig:tsne_mae_a}
  \end{subfigure}
  \begin{subfigure}[b]{0.32\textwidth}
    \centering
    \includegraphics[width=\textwidth]{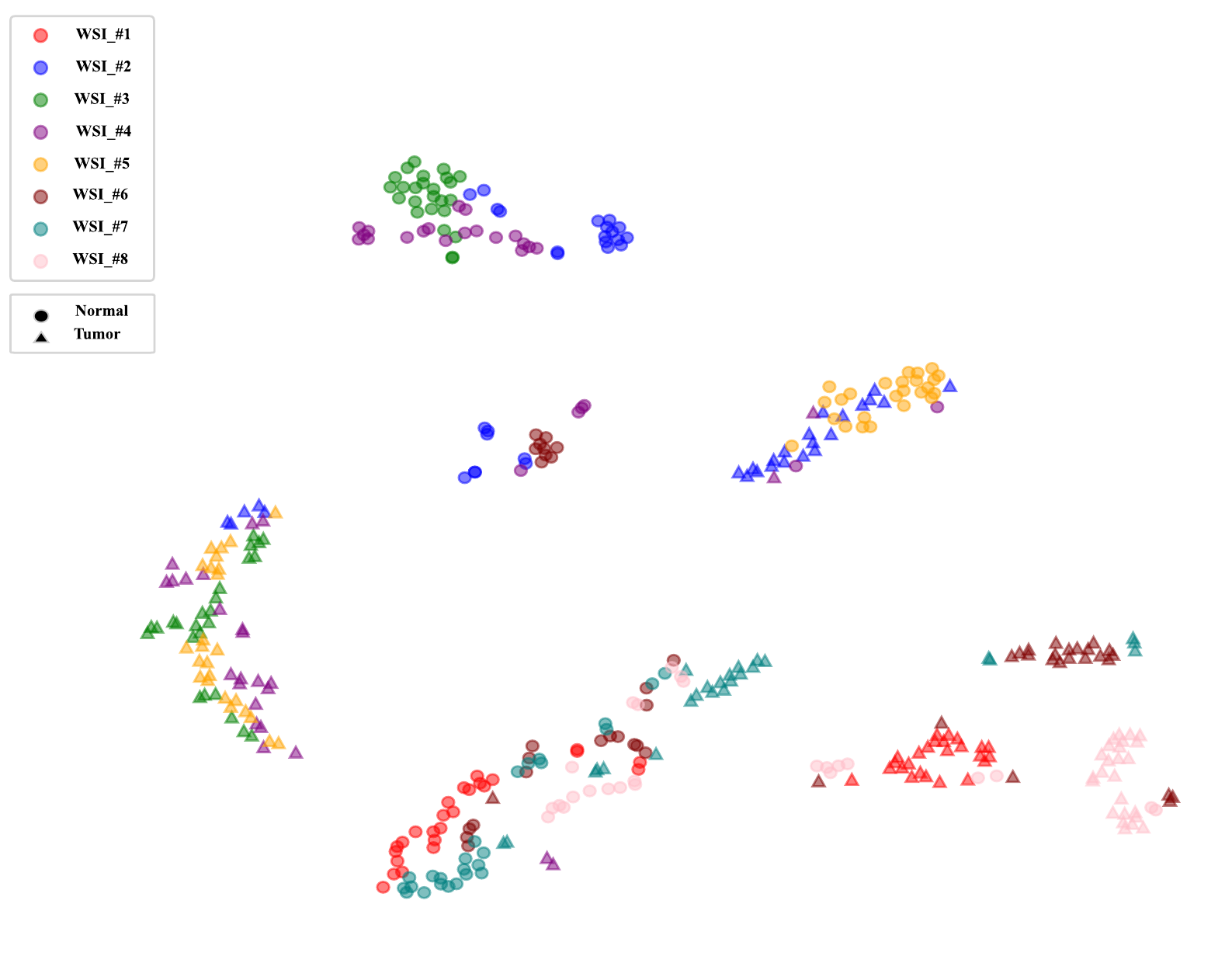}
    \caption{} \label{fig:tsne_mae_b}
  \end{subfigure}
  \begin{subfigure}[b]{0.32\textwidth}
    \centering
    \includegraphics[width=\textwidth]{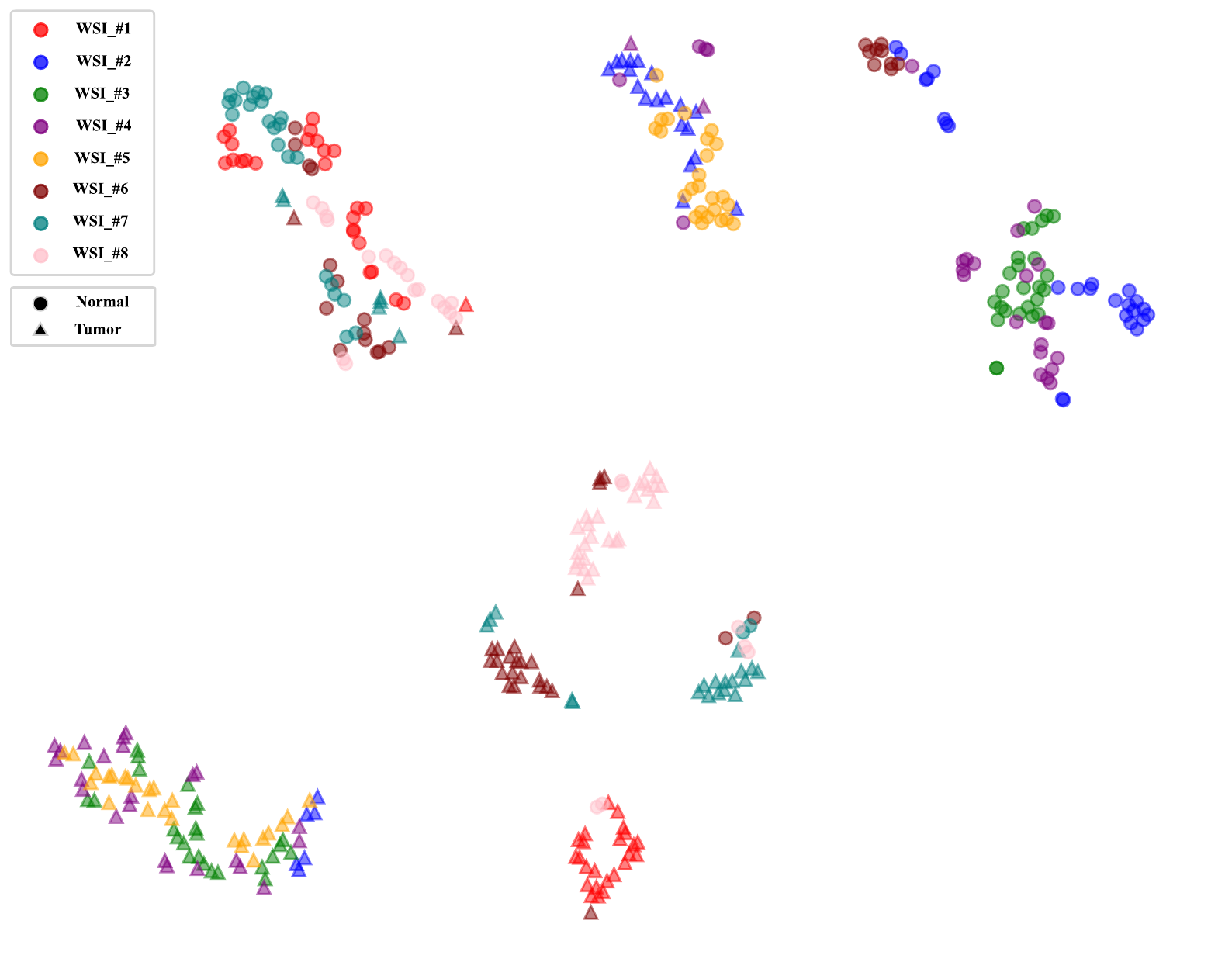}
    \caption{} \label{fig:tsne_mae_c}
  \end{subfigure}
  \caption{t-SNE of MAE learned embeddings. (a) on ImageNet-1K, (b) Full-tuning from ImageNet-1K to CAMAELYON16 \cite{camelyon16}, and (c) Adapter tuning from ImageNet-1K to CAMELYON16 \cite{camelyon16}.}
  \label{fig:IDK}
\end{figure*}

\section{t-SNE of Snuffy Embeddings}
\label{tsne_snuffy}

Here, we provide the t-SNE of learned embeddings of the Attention-pooling in Snuffy for the CAMELYON16 and TCGA Lung Cancer datasets \cite{camelyon16, tcga}. As we can see, the model shows a clear distinction between different classes.

\begin{figure}
    \centering
    \begin{minipage}{0.49\textwidth}
        \centering
        \includegraphics[width=1\textwidth]{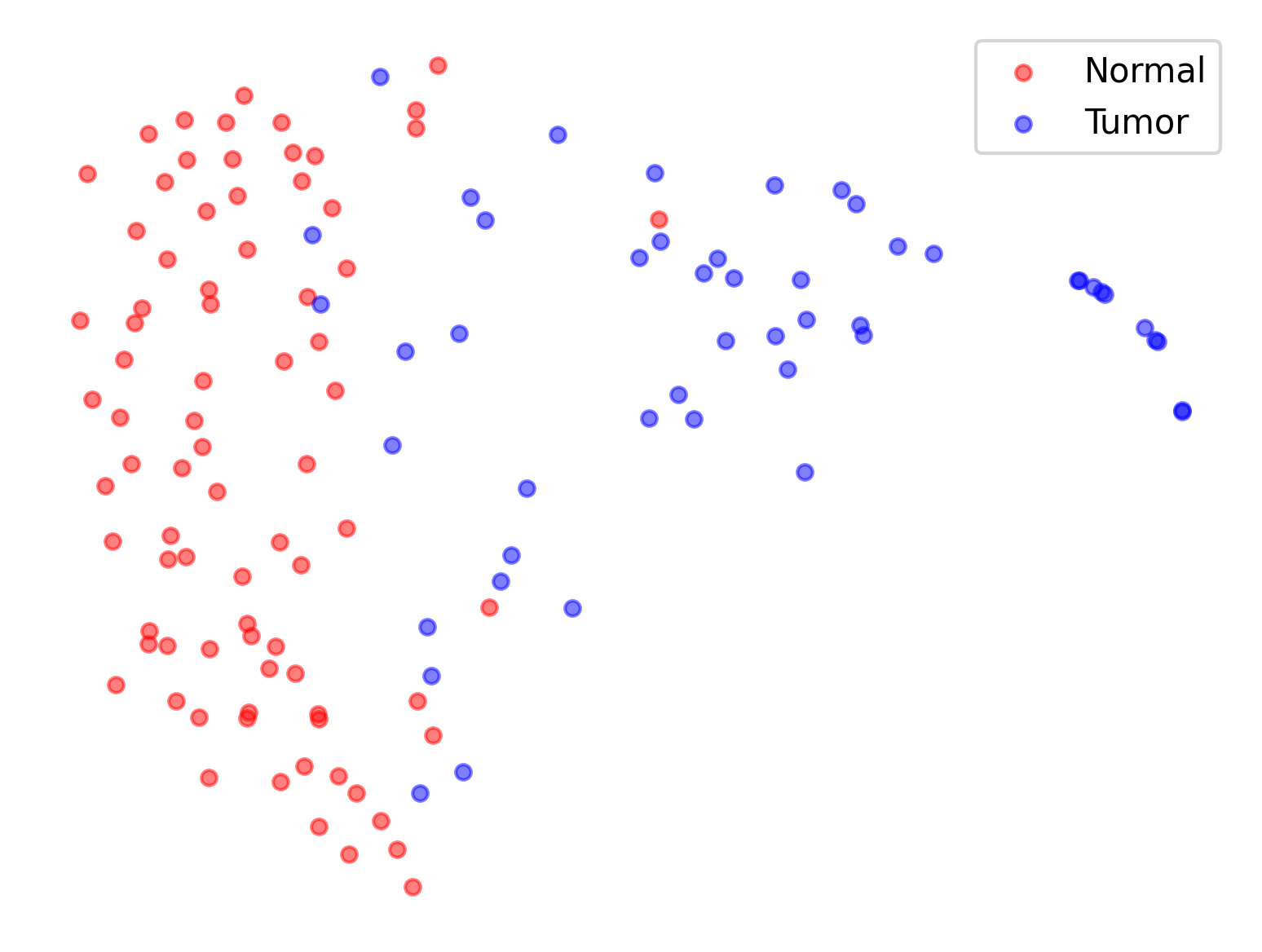} 
        \caption{t-SNE of learned embeddings on the CAMELYON16 \cite{camelyon16} dataset.}
        \label{fig:camtsne}
    \end{minipage}\hfill
    \begin{minipage}{0.49\textwidth}
        \centering
        \includegraphics[width=1\textwidth]{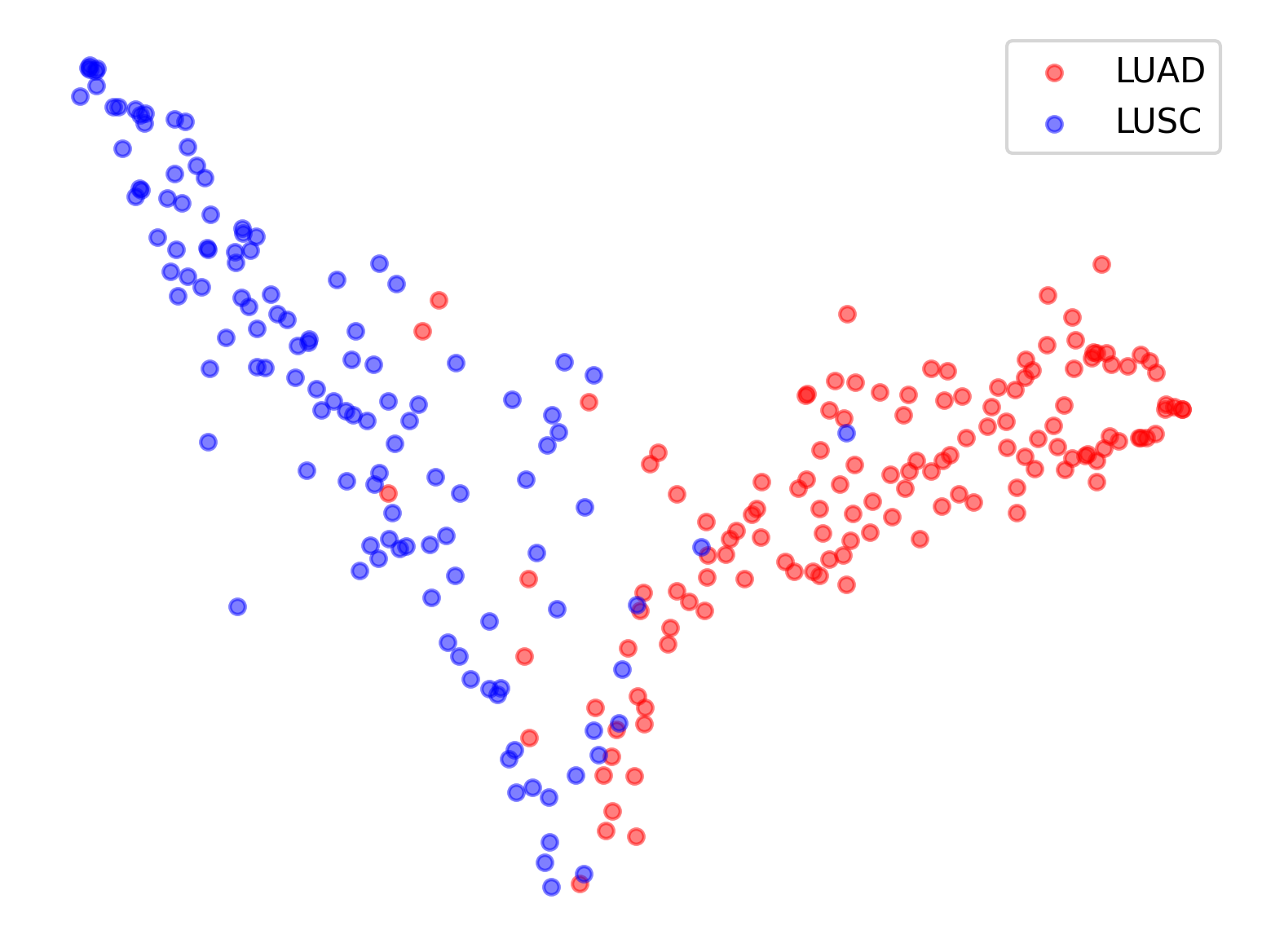} 
        \caption{t-SNE of learned embeddings on the TCGA Lung Cancer dataset \cite{tcga}.}
        \label{fig:tcgatsne}
    \end{minipage}
\end{figure}

\section{Additional ROI Detection Images}
\label{additional_roi}

We provide additional examples of ROI detection on the CAMELYON16 and TCGA Lung Cancer datasets \cite{camelyon16, tcga}. Illustrated in Fig. \ref{fig:additional_roi_camelyon}, our model adeptly identifies cancerous regions, even when such areas comprise small portions of the images. Regarding the TCGA Lung Cancer dataset, depicted in Fig. \ref{fig:additional_roi_tcga}, ground truth for cancerous regions is not available. However, it is known that generally, more than \( 80\% \) of the images in this dataset contain cancer, and our model's performance is consistent with this statistic.

\begin{figure*}
  \centering
  \hspace*{7mm}\begin{subfigure}[b]{0.4\textwidth}
    \centering
    \includegraphics[width=\textwidth]{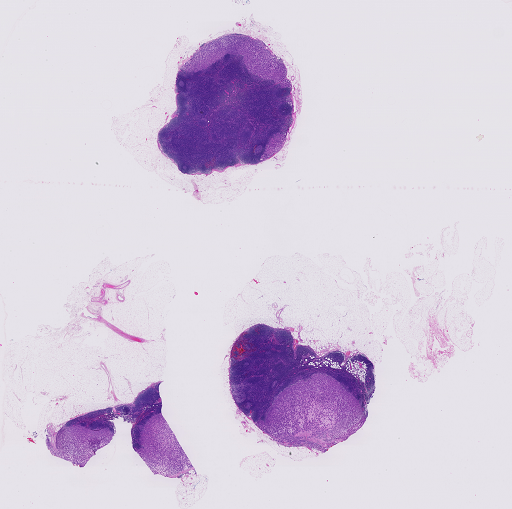}
  \end{subfigure}
  \begin{subfigure}[b]{0.4\textwidth}
    \centering
    \includegraphics[width=\textwidth]{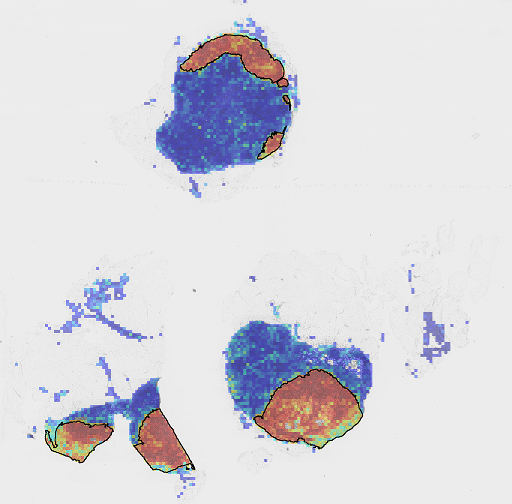}
  \end{subfigure}
  \begin{subfigure}[b]{0.1\textwidth} 
    \raisebox{0.1mm}{
      \includegraphics[width=0.35\textwidth,height=48mm]{figures/heatmap_legend.png}
      }
  \end{subfigure}
  \hspace*{7mm}\begin{subfigure}[b]{0.4\textwidth}
    \centering
    \includegraphics[width=\textwidth]{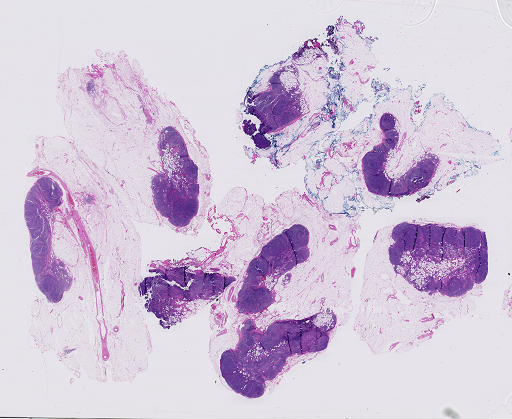}
  \end{subfigure}
  \begin{subfigure}[b]{0.4\textwidth}
    \centering
    \includegraphics[width=\textwidth]{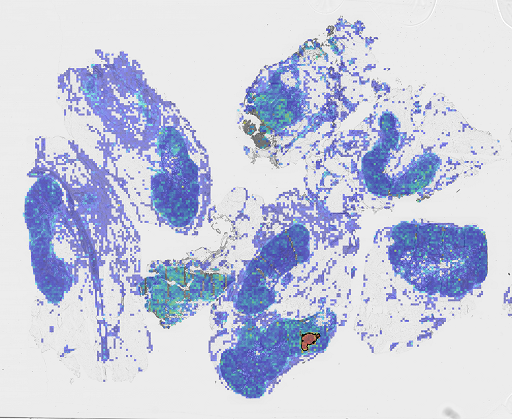}
  \end{subfigure}
  \begin{subfigure}[b]{0.1\textwidth} 
    \raisebox{0.1mm}{
      \includegraphics[width=0.35\textwidth,height=40mm]{figures/heatmap_legend.png}
      }
  \end{subfigure}
  \hspace*{7mm}\begin{subfigure}[b]{0.4\textwidth}
    \centering
    \includegraphics[width=\textwidth]{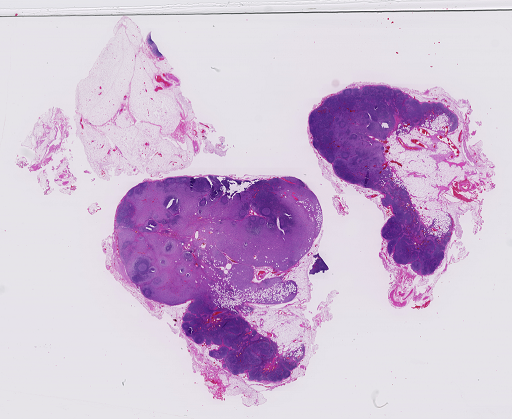}
  \end{subfigure}
  \begin{subfigure}[b]{0.4\textwidth}
    \centering
    \includegraphics[width=\textwidth]{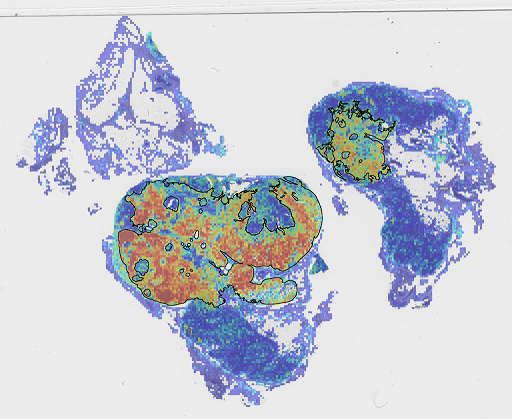}
  \end{subfigure}
  \begin{subfigure}[b]{0.1\textwidth} 
    \raisebox{0.1mm}{
      \includegraphics[width=0.35\textwidth,height=40mm]{figures/heatmap_legend.png}
      }
  \end{subfigure}
  \hspace*{7mm}\begin{subfigure}[b]{0.4\textwidth}
    \centering
    \includegraphics[width=\textwidth]{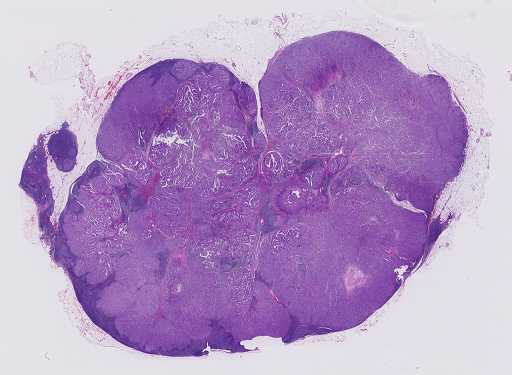}
    \caption{} \label{fig:additional_roi_camelyon_a}
  \end{subfigure}
  \begin{subfigure}[b]{0.4\textwidth}
    \centering
    \includegraphics[width=\textwidth]{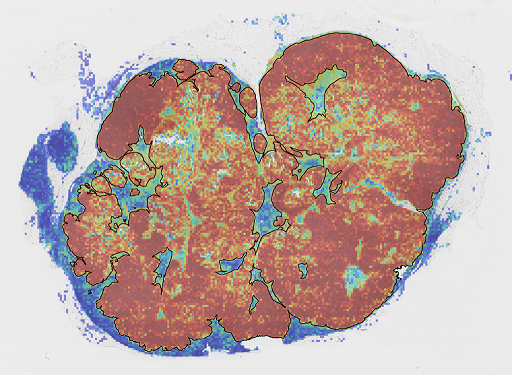}
    \caption{} \label{fig:additional_roi_camelyon_b}
  \end{subfigure}
  \begin{subfigure}[b]{0.1\textwidth} 
    \raisebox{3.9mm}{
      \includegraphics[width=0.35\textwidth,height=36mm]{figures/heatmap_legend.png}
      }
  \end{subfigure}
  \caption{Additional Images for qualitative assessment of ROI detection in Snuffy on the CAMELYON16 \cite{camelyon16} dataset. All Images are labaled as Tumor and the black line shows the ground truth.}
  \label{fig:additional_roi_camelyon}
\end{figure*}
\begin{figure*}[!htbp]
  \centering 
  \hspace*{7mm}\begin{subfigure}[b]{0.39\textwidth}
    \centering
    \includegraphics[width=\textwidth]{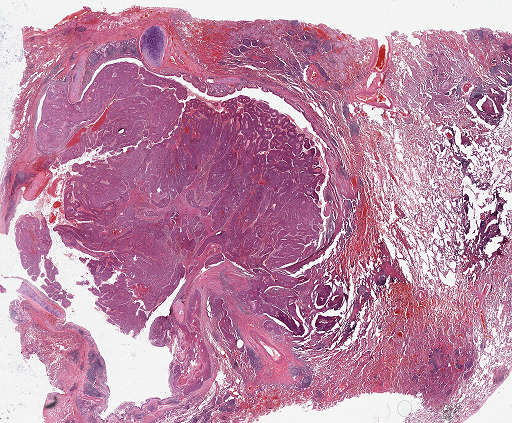}
  \end{subfigure}
  \begin{subfigure}[b]{0.39\textwidth}
    \centering
    \includegraphics[width=\textwidth]{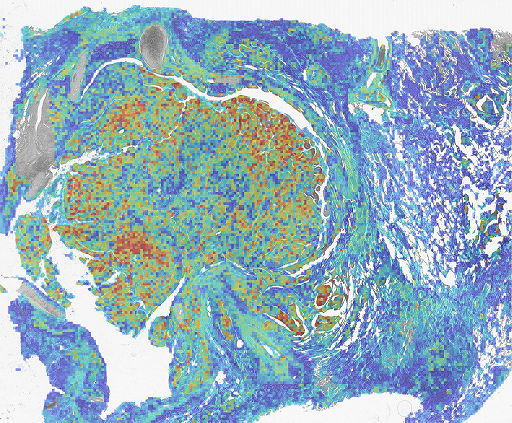}
  \end{subfigure}
  \begin{subfigure}[b]{0.1\textwidth} 
    \raisebox{0.1mm}{
      \includegraphics[width=0.35\textwidth,height=39mm]{figures/heatmap_legend.png}
      }
  \end{subfigure}
  \hspace*{7mm}\begin{subfigure}[b]{0.39\textwidth}
    \centering
    \includegraphics[width=\textwidth]{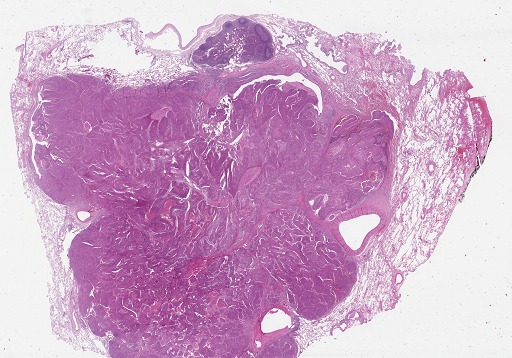}
  \end{subfigure}
  \begin{subfigure}[b]{0.39\textwidth}
    \centering
    \includegraphics[width=\textwidth]{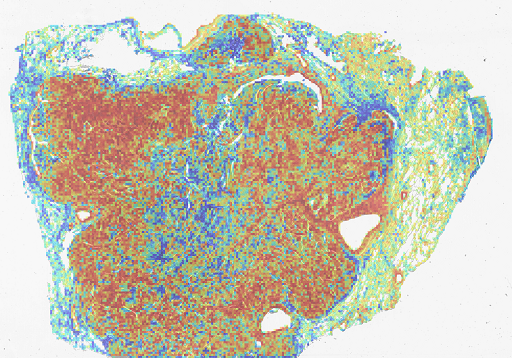}
  \end{subfigure}
  \begin{subfigure}[b]{0.1\textwidth} 
    \raisebox{0.1mm}{
      \includegraphics[width=0.35\textwidth,height=33mm]{figures/heatmap_legend.png}
      }
  \end{subfigure}
  \hspace*{7mm}\begin{subfigure}[b]{0.39\textwidth}
    \centering
    \includegraphics[width=\textwidth]{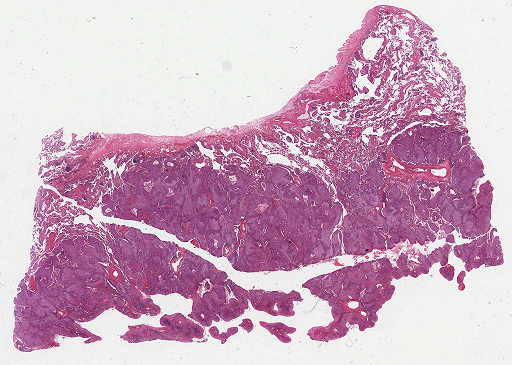}
  \end{subfigure}
  \begin{subfigure}[b]{0.39\textwidth}
    \centering
    \includegraphics[width=\textwidth]{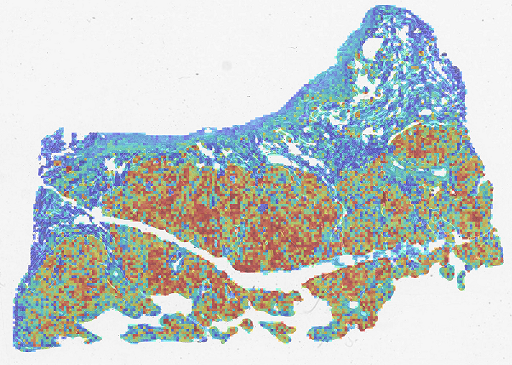}
  \end{subfigure}
  \begin{subfigure}[b]{0.1\textwidth} 
    \raisebox{0.1mm}{
      \includegraphics[width=0.35\textwidth,height=34.2mm]{figures/heatmap_legend.png}
      }
  \end{subfigure}
  \hspace*{7mm}\begin{subfigure}[b]{0.39\textwidth}
    \centering
    \includegraphics[width=\textwidth]{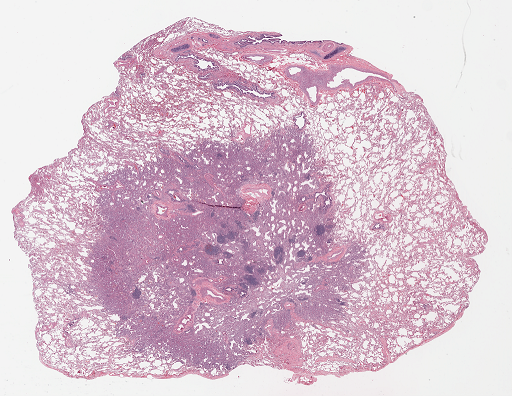}
  \end{subfigure}
  \begin{subfigure}[b]{0.39\textwidth}
    \centering
    \includegraphics[width=\textwidth]{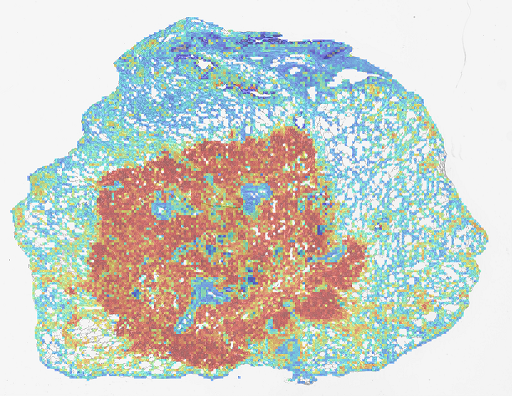}
  \end{subfigure}
  \begin{subfigure}[b]{0.1\textwidth} 
    \raisebox{0.1mm}{
      \includegraphics[width=0.35\textwidth,height=36.5mm]{figures/heatmap_legend.png}
      }
  \end{subfigure}
  \hspace*{7mm}\begin{subfigure}[b]{0.39\textwidth}
    \centering
    \includegraphics[width=\textwidth]{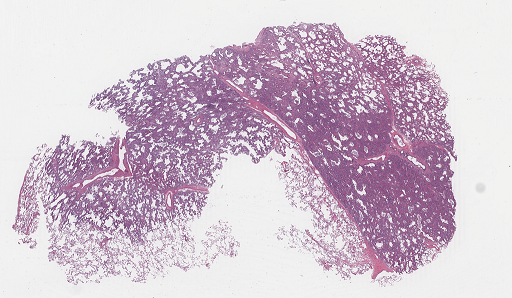}
    \caption{} \label{fig:additional_roi_tcga_a}
  \end{subfigure}
  \begin{subfigure}[b]{0.39\textwidth}
    \centering
    \includegraphics[width=\textwidth]{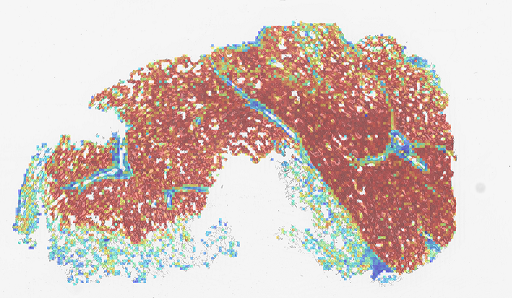}
    \caption{} \label{fig:additional_roi_tcga_b}
  \end{subfigure}
  \begin{subfigure}[b]{0.1\textwidth} 
    \raisebox{3.9mm}{
      \includegraphics[width=0.35\textwidth,height=27.5mm]{figures/heatmap_legend.png}
      }
  \end{subfigure}
  \caption{Additional Images for qualitative assessment of ROI detection in Snuffy on the TCGA Lung Cancer dataset \cite{tcga}. The top three rows are labeled as LUSC, and the bottom two are labeled as LUAD.}
  \label{fig:additional_roi_tcga}
\end{figure*}

\section{Analysis of Layer Counts for Universal Approximation}
An examination of the relationship between the number of layers per $\lambda_r$ and the quantity of patches $n$ is presented (see Fig.~\ref{fig:abl_layers}). It is evident that our probability constraint on number of layers significantly surpasses the current limitation on layer count imposed by \cite{bigbird} for ensuring universal approximation.

\begin{figure}[!htbp]
        \centering
\begin{subfigure}{\textwidth}
    \centering
    \includegraphics[width=0.745\textwidth]{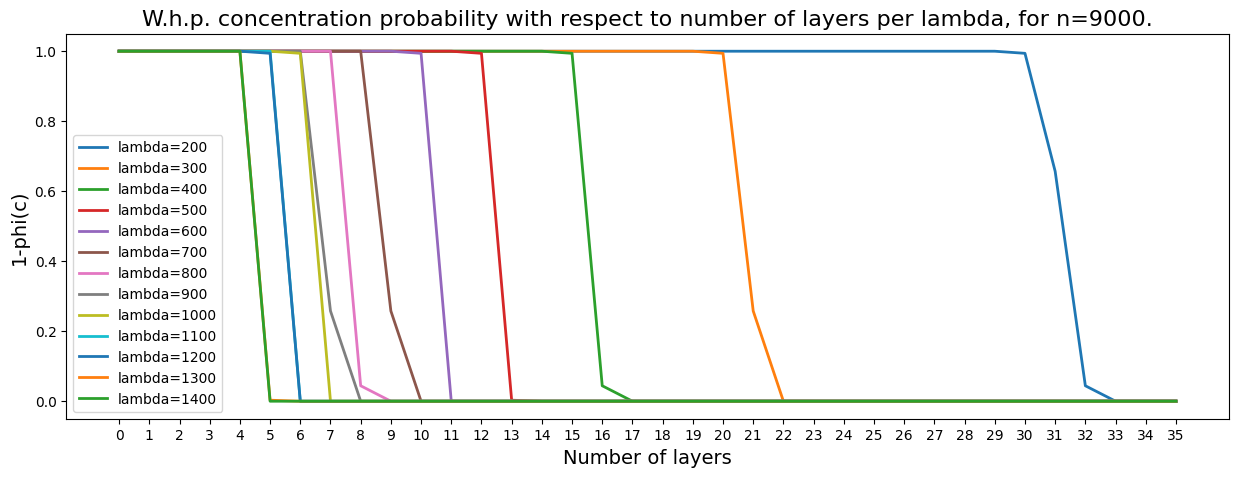}
    \caption{Comparison of convergence of w.h.p in layers per different $\lambda_r$ for number of patches $n=9000$.}
    \label{fig:first}
\end{subfigure}
\hfill
\hfill
\begin{subfigure}{0.7\textwidth}
    \includegraphics[width=0.745\textwidth]{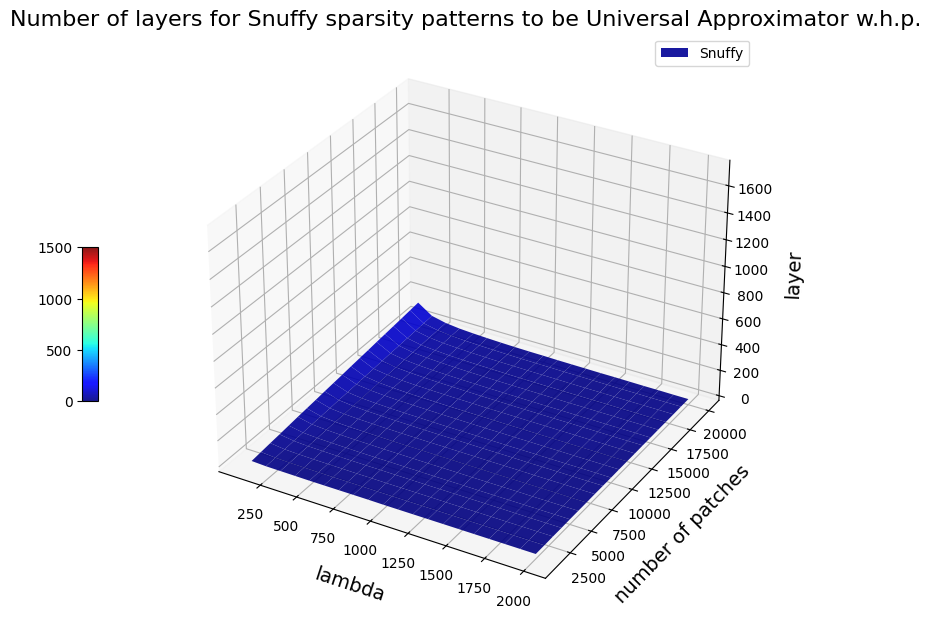}
    \caption{Minimum number of layer for Snuffy sparsity pattern to be universal approximator w.h.p.}
    \label{fig:second}
\end{subfigure}
\hfill
\hfill
\begin{subfigure}{0.7\textwidth}
    \includegraphics[width=0.745\textwidth]{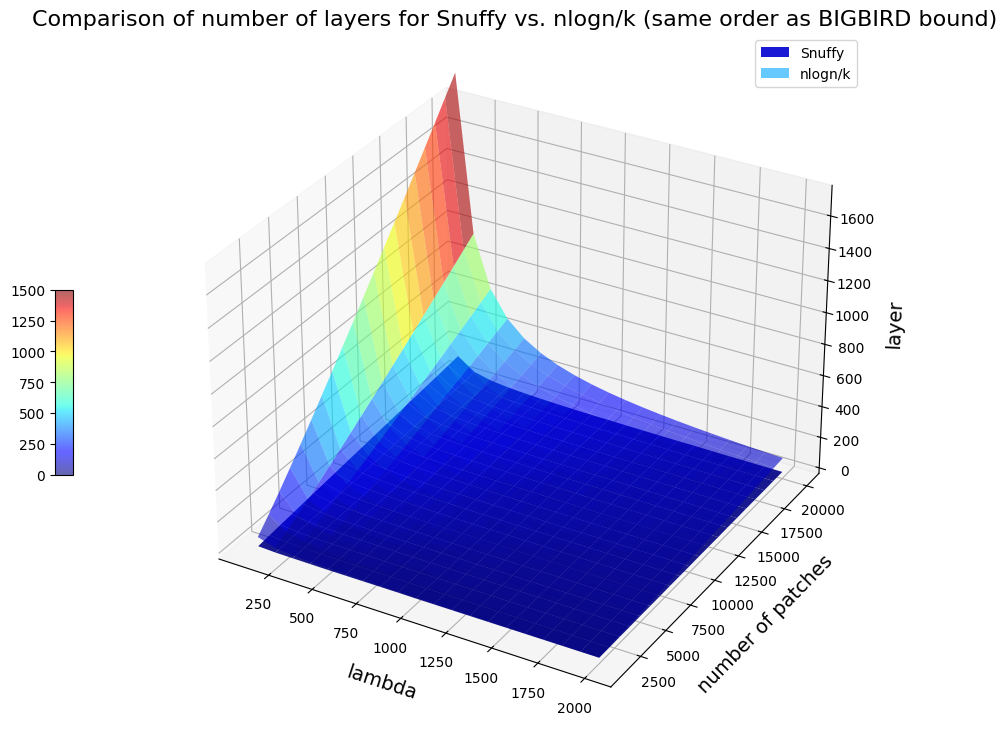}
    \caption{Comparison of Snuffy bound on number of layers to the $\frac{n\log n}{k}$}
    \label{fig:third}
\end{subfigure}
        
\caption{Analysis on number of layers per $\lambda_r$ and number of patches $n$. \textbf{(a)} For a fixed $n=9000$ convergence of $1-\Phi(c)$ per different $\lambda_r$ is depicted. \textbf{(b)} 3D illustration on minimum number of layers needed for Snuffy to be universal approximator. \textbf{(c)} Comparison of the same plot (the lower surface) with surface $\frac{n\log n}{k}$ (the upper surface) as suggested by previous literature \cite{bigbird}. }
\label{fig:abl_layers}
\end{figure}



\end{document}